\documentclass[11pt]{article}
\pdfoutput=1

\usepackage{graphicx}

\usepackage{enumerate}
\usepackage[OT1]{fontenc}
\usepackage{amsmath,amssymb}
\usepackage{subfigure}
\usepackage{natbib}
\usepackage[usenames]{color}
\usepackage{multirow}
\usepackage[colorlinks,
            linkcolor=red,
            anchorcolor=blue,
            citecolor=blue
            ]{hyperref}

\usepackage{hyperref}

\usepackage{enumitem}
\usepackage{mathrsfs}
\usepackage{fullpage}
\usepackage[protrusion=false,expansion=true]{microtype}
\usepackage{smile}

\usepackage{mathtools}
\usepackage{amsthm}

\usepackage{romannum}

\usepackage{mdframed}

\newenvironment{Msg}[1]
  {\mdfsetup{
    frametitle={\colorbox{white}{\space \large #1\space}},
    innertopmargin=-3pt,
    innerbottommargin=7pt,
    innerrightmargin=7pt,
    innerleftmargin=7pt,
    frametitleaboveskip=-\ht\strutbox,
    frametitlealignment=\center,
    linewidth=1pt
    }
  \begin{mdframed}%
  }
{\end{mdframed}}

\theoremstyle{plain}
\newtheorem{theorem}{Theorem}[section]

\newtheorem{lemma}[theorem]{Lemma}
\newtheorem{corollary}[theorem]{Corollary}
\newtheorem{claim}[theorem]{Claim}
\theoremstyle{definition}
\newtheorem{definition}[theorem]{Definition}

\newtheorem{property}[theorem]{Property}
\theoremstyle{remark}

\newtheorem{hypothesis}[theorem]{\textbf{Induction Hypothesis}}
\newtheorem{fact}[theorem]{Fact}
\def \polylog {\mathrm{polylog}}

\def \la {\langle}
\def \ra {\rangle}
\def \poly {\mathrm{poly}}

\def\su{\textit{Supp}}

\usepackage[textsize=tiny]{todonotes}

\begin{document}
\pagenumbering{arabic}

\title{\bf Modality Competition:  What Makes Joint Training of Multi-modal Network Fail in Deep Learning? (Provably)}
\author
{
    Yu Huang\thanks{IIIS, Tsinghua University; e-mail: {\tt y-huang20@mails.tsinghua.edu.cn}}
    ~~~
    Junyang Lin\thanks{Alibaba Group; e-mail: {\tt junyang.ljy@alibaba-inc.com}}
    ~~~
      Chang Zhou \thanks{Alibaba Group; e-mail: {\tt ericzhou.zc@alibaba-inc.com}}
     ~~~
      Hongxia Yang \thanks{Alibaba Group; e-mail: {\tt yang.yhx@alibaba-inc.com}}
      ~~~
      Longbo Huang\thanks{IIIS, Tsinghua University; e-mail: {\tt longbohuang@tsinghua.edu.cn}}
}

\date{}
\maketitle

\begin{abstract}
    Despite the remarkable success of deep multi-modal learning in practice, it has not been well-explained in theory. 
    Recently, it has been observed that the best uni-modal network outperforms the jointly trained multi-modal network 
    , which is counter-intuitive since multiple signals generally bring more information~\cite{wang2020makes}.
    This work provides a theoretical explanation for the emergence of such performance gap in neural networks for the prevalent joint training framework. 
    Based on a simplified data distribution that captures the realistic property of multi-modal data, we prove that for the multi-modal late-fusion network with (smoothed) ReLU activation trained jointly by gradient descent, different modalities will compete with each other. 
    The encoder networks will learn only a subset of modalities. 
    We refer to this phenomenon as modality competition. 
    The losing modalities, which fail to be discovered, are the origins where the sub-optimality of joint training comes from.
  Experimentally, we illustrate that modality competition  matches the intrinsic behavior of late-fusion joint training. 
\end{abstract}

\section{Introduction}
\label{sec1}
Deep multi-modal learning has achieved remarkable performance in a wide range of fields, such as speech recognition~\cite{Chan2016ListenAA}, semantic segmentation~\cite{jiang2018rednet}, and visual question-answering (VQA)~\cite{anderson2018vision}. 
Intuitively, signals from different modalities often provide complementary information leading to performance improvement. 
However, \citet{wang2020makes} observed that the best uni-modal network outperforms the multi-modal network obtained by joint training. 
Moreover, the analogous phenomenon has been noticed when using multiple input streams~\cite{goyal2017making,gat2020removing,alamri2019audio}.

Although deep multi-modal learning has become a critical practical machine learning approach, its theoretical understanding is quite limited. 
Some recent works have been proposed for understanding multi-modal learning from a theoretical standpoint~\cite{zhang2019cpm,huang2021makes, sun2020tcgm, du2021modality}. \citet{huang2021makes} provably argues that the generalization ability of uni-modal solutions is strictly sub-optimal than  that of multi-modal solutions. 
\citet{du2021modality} aims at identifying the reasons behind the surprising phenomenon of performance drop. Remarkably, these works have not analyzed what happened in the \emph{training} process of \emph{neural networks}, which we deem as crucial to understanding why naive joint training fails in practice. 
In particular, 
 we state the fundamental questions that we address below and provably answer these questions by studying a simplified data model  
 that captures key properties of real-world settings under the popular late-fusion joint training framework~\cite{baltruvsaitis2018multimodal}. We provide empirical results to support our theoretical framework. Our work is the first theoretical treatment towards the degenerating aspect of multi-modal learning in neural networks to the best of our knowledge.
 

\begin{Msg}{Fundamental Questions}
  \it  1. How does the neural network encoder of each modality, trained by multi-modal learning, learn its feature representation?\\
    2. Why does multi-modal learning {in deep learning} collapse in practice when naive joint training is applied?
\end{Msg}


\subsection{Our Contributions}
We study the multi-label classification task for a data distribution where each modality $\mathcal{M}_r$ is generated from a sparse coding model,  
which shares similarities with real scenarios (formally presented and explained in Section~\ref{sec2}). Our data model for each modality owns a special structure called ``insufficient data," which represents cases where each modality alone cannot adequately predict the task. Such a structure is common  in practical multi-modal applications~\cite{yang2015auxiliary,liu2018learn, gat2020removing}. 
Under this data model, we consider joint training based on late-fusion multi-modal network with one-layer neural network,  activated by smoothed ReLU as modality encoder, and features from different modalities are passed to one-layer linear classifier after being fused by sum operation. 
Comparatively, the uni-modal network has similar pattern with the fusion operation eliminated. Both networks are trained by gradient descent (GD) over the multi-modal training set $\mathcal{D}$ or its uni-modal counterpart $\mathcal{D}^{r}$. 

We analyze the optimization and generalization of multi and uni-network to probe the origin of the gap between theory and practice of multi-modal joint training in deep learning. Our key theoretical findings are summarized as follows. 
\begin{itemize}
    \item When only single modality is applied to  training, the uni-modal network will focus on learning the modality-associated features, which  leads to good performance (Theorem~\ref{thm-uni}).  
    \item When naive joint training is applied to the multi-modal network, the neural network will not efficiently learn all features from different modalities, and only a subset of modality encoders will capture sufficient feature representations (Theorem~\ref{thm-mul}).  We call this process ``Modality Competition" and sketch its high-level idea below. 
    \begin{Msg}{Modality Competition} 
    During joint training, multiple modalities will compete with each other. Only a \emph{subset} of modalities which correlate more with their encoding network's random initialization will win and be learned by the final modality with other modalities failing to be explored.
    \end{Msg}
    \item With the different feature learning process and 
    the existence of insufficient structure, we further establish the theoretical guarantees for performance gap measured by test error, between the uni-modal and multi-modal networks (Corollary~\ref{col}). 
\end{itemize}

\paragraph{Empirical justification:}

\begin{figure*}[!ht]
\vskip 0.2in
    \centering
    \includegraphics[width=\linewidth]{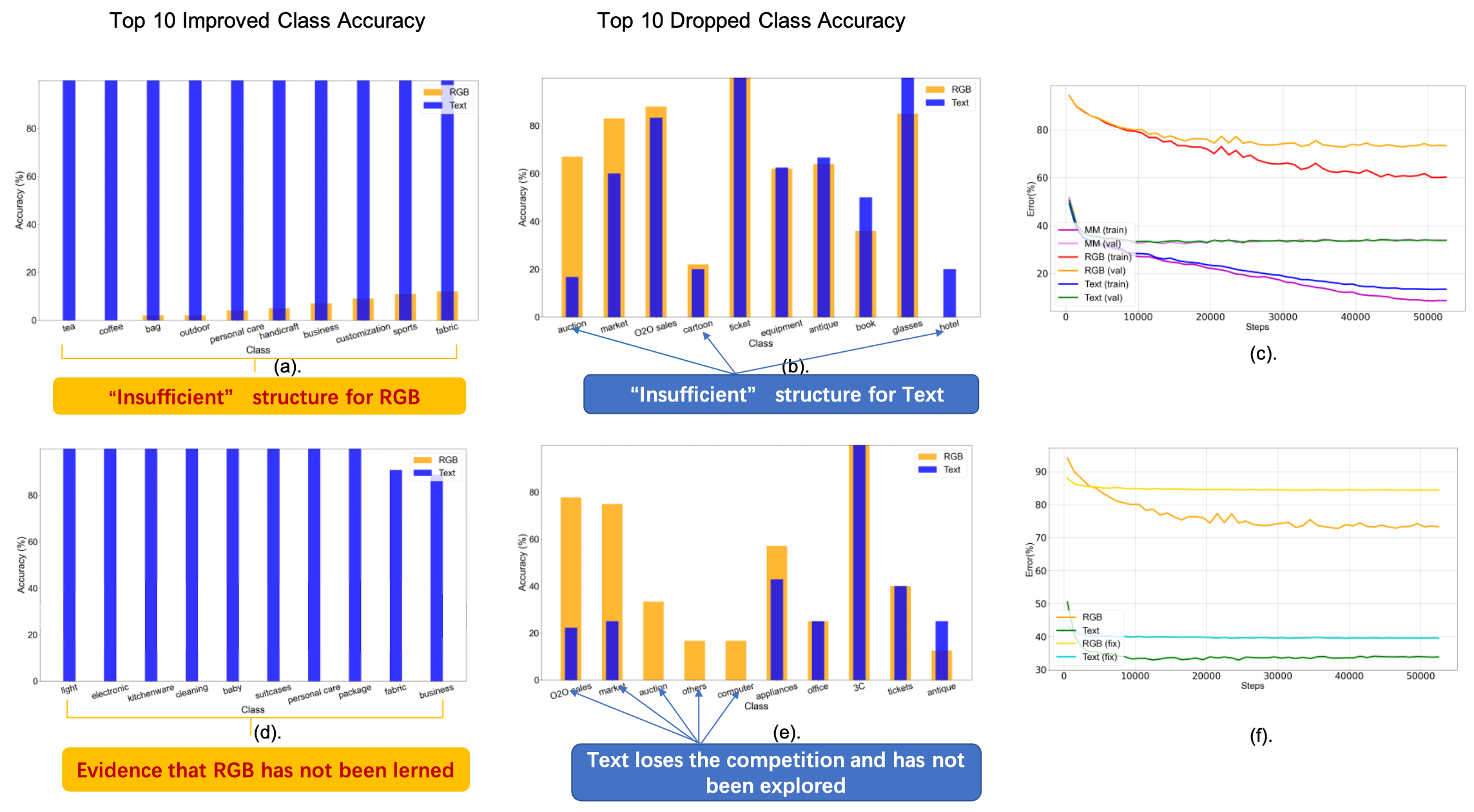}
    \vskip -0.05in
    \caption{We experiment on item classification with the setups of image(RGB)-only, text-only, and multi-modality with joint training. Detailed setups are provided in Appendix~\ref{sec-exp}.  
    (a) and (b) report the top $10$ classes based on the accuracy improvement and downgrade of text-only over image-only uni-model. (c) illustrates the training and validation error curves for text-only, image-only and text+image models.  (d) and (e) demonstrate the similar comparison as (a) and (b) for the ones with a fixed encoder initialized by the multi-modal joint training. (f) illustrates the error curves for the directly trained uni-modal models and the ones with a fixed encoder. 
    }
    \label{fig:exp}
\end{figure*}

We also support our findings with empirical results. 
\begin{itemize} 
    \item   For each modality, there exist certain classes where the corresponding uni-modal network has relatively low accuracy as shown in Figure~\ref{fig:exp} (a) and (b). For example, as demonstrated in Figure~\ref{fig:exp} (b), for text modality, while it predicts well on most classes, there exist some classes e.g. "auction", where it has low accuracy. Such observations verify the  insufficient  structure of uni-modal data.
    \item  Figure~\ref{fig:exp} (c) supports the findings in  \cite{wang2020makes} that the best uni-modal outperforms the multi-modal. 
    \item Only a subset of modalities learns good feature representations. 
    As illustrated in Figure~\ref{fig:exp} (d), for some classes, e.g., ``fabric", ``business" that were originally with slightly high accuracy (from (a)), the accuracy still drops to zero, which indicates that images are not learned for these classes in joint training. We have similar observations for text modality by comparing Figure~\ref{fig:exp} (b) and (e). Moreover, Figure~\ref{fig:exp} (f) shows that the feature representations obtained from joint training for each modality degrade compared to directly trained uni-modal.
\end{itemize}
The rest of the paper is organized as follows. We discuss the literature most related to our work in Section~\ref{sec2}. In Section~\ref{sec3}, we introduce the problem setup of
our work. Main theoretical results are provided in Section~\ref{sec4}. We present the main intuition and sketch of our proof in Section~\ref{sec5}. We conclude our work and discuss some future work in Section~\ref{sec6}.

\section{Related Work}\label{sec2}
\paragraph{Success of multi-modal application.} With the development of deep learning, combining different modalities (text, vision, etc.) to solve the tasks has become a common approach in machine learning approach, and have demonstrated great power in various applications. 
Achievements have been made on tasks, which it is insufficient for single-modal models to learn, e.g., speech recognition~\cite{schneider2019wav2vec,dong2018speech}, sound localization~\cite{zhao2019sound} and VQA~\cite{anderson2018vision}. 
On the other hand, a large body of studies in vision \& language learning~\cite{chen2020uniter, li2020oscar, li2020unicoder, m6} use pre-trained  encoders to extract features from different modalities. These studies which demonstrate the success of multi-modal learning are beyond the scope of our research. Instead, in this paper, we focus on the end-to-end late-fusion multi-modal network with different modalities trained jointly and aim to theoretically explore the commonly observed phenomenon~\cite{wang2020makes} in this setting that multi-modal network does not make performance improvement over best uni-modal.
\paragraph{Theory of Multi-modal Learning}
Theoretical progress in understanding multi-modal learning has lagged. Existing analysis 
for multi-view learning~\cite{xu2013survey, amini2009learning, federici2020learning}, which is similar to multi-modal learning, does not readily generalize to multi-modal settings. It typically assumes that each view alone is sufficient to predict the target accurately, which is problematic in our settings, since in some cases we cannot make accurate decisions only with a single-modality (e.g., depth image for object detection~\cite{gupta2016cross}). One sequence of theoretical works try to explain the advantages of multi-modal using information-theoretical framework~\cite{sun2020tcgm}
or assuming the training process is perfect~\cite{huang2021makes,zhang2019cpm}. Recently, \cite{du2021modality} utilized the easy-to-learn and paired features to explain the failure of joint training. However, their results do not take neural network architecture into consideration and do not provide the analysis of training process. Although these theoretical works shed great lights to the study of multi-modal learning, they have not yet given concrete mathematical answers to the fundamental questions we asked earlier.
\paragraph{Feature learning by neural networks.}
In recent years, there has been an interest in studying the \textit{feature learning process} of neural networks. \citet{allen2020towards} contribute to understanding how ensemble and knowledge distillation work in deep learning based on a generic ``multi-view" feature structure. 
\citet{wen2021toward} prove that contrastive learning with proper data augmentation can learn desired sparse features resembling the features learning in supervised setting. Our proof techniques and intuitions are related to these recent literature, and our work studies 
a different perspective of feature learning by multi-modal joint training. 

\section{Notations}
 $[K]$ denotes the index set $\{1, \ldots,K\} .$ For a matrix $\mathbf{M}$, we use $\mathbf{M}_{j}$ to denote its $j$-th column. For a vector $x=(x_1,\cdots,x_d)^{\top}$, $\|x\|_0$ denotes the number of its non-zero elements and $\|x\|_{\infty}:=\max_{j\in d}|x_j|$. 
We use the standard big-O notation and its variants: $\mathcal{O}(\cdot), o(\cdot), \Theta(\cdot), \Omega(\cdot), \omega(\cdot)$, where $K$ is the problem parameter that becomes large. Occasionally, we use the symbol $\widetilde{\mathcal{O}}(\cdot)$ (and analogously with the other four variants) to hide $\operatorname{polylog}(K)$ factors. \textit{w.h.p} means with probability at least $1-e^{-\Omega(\log^2(K))}$.    $\su(\cdot)$ denotes the support of a random variable.

\section{Problem Setup}
\label{sec3}

We present our formulation, including the data distribution and learner network. We focus on a multi-class classification problem. 

\subsection{Data distribution:} 

Let $\Xb$ be a data sample and $y\in [K]$ be the corresponding label. 
For simplicity, we consider $\Xb:= (\Xb^{1}, \Xb^{2})$ consisting of two modalities,\footnote{Our setting can be easily generalized to multiple modalities at the expense of complicating notations.} and each modality $\mathcal{M}_r$, $r\in [2]$, is associated with a vector $\Xb^{r}\in \RR^{d_r}$. We assume that the raw data is generated from a sparse coding model:
\begin{align*}
    \Xb^{1}=\mathbf{M}^{1} z^{1}+\xi^{1},\quad & \Xb^{2}=\mathbf{M}^{2} z^{2}+\xi^{2}\\
    (z^{1},z^{2})\sim\mathcal{P}_z\quad  &\xi^{r}\sim\mathcal{P}_{\xi^{r}}\text{ for } r\in[2]
\end{align*}

for dictionary $\mathbf{M}^{r} \in \mathbb{R}^{d_r \times K}$, where $z^{r} \in \mathbb{R}^{K}$ is the sparse vector 
and $\xi^{r} \in \mathbb{R}^{d_r}$ is the noise.  There are three main components $\Mb^{r}$, $z^{r}$, $\xi^{r}$, and we will introduce them in detail below. For simplicity, we focus on the case where 
$\mathbf{M}^{1},\mathbf{M}^{2}$ are unitary with orthogonal columns.

\paragraph{Why sparse coding model?} Our data model shares many similarities with practical scenarios. 
Originated to explaining neuronal activation of human visual  system~\cite{OLSHAUSEN19973311}, sparse coding model has been widely used in  machine learning
applications to model different uni-modal data, such as image, text and audio~\cite{mairal2010online,yang2009linear,yogatama2015learning, arora2018linear,  whitaker2016heart,grosse2012shift}.
Also, there is a line of research to develop sparse representations for multiple modalities simultaneously~\cite{yuan2012visual,shafiee2015multi,gwon2016multimodal}. 

In our following descriptions, we specify the choices for parameters including  $\gamma_r,s, \alpha$ for the sake of clarity. Our results apply to a wider range of parameters and generalized details are provided in Appendix~\ref{sec-not}. 

\paragraph{Distribution of sparse vector:} We generate $(z^1,z^2)$ from the joint distribution $\cP_z$ as follows: 

a). Select the label $y\in[K]$ uniformly at random;

b). Given the label $y$, the distribution $\mathcal{P}_{z^r\mid y}$ for each modality $\mathcal{M}_r$ is divided into two categories:
\begin{itemize}
    \item  With probability $\mu_r=\frac{1}{\poly(K)}$, $z^r$ is generated from the  \textit{insufficient} class:
    \begin{itemize}
        \item $z^{r}_{y}=\Theta(\gamma_r)$, we assume $\gamma_1=\gamma_2=\frac{1}{K^{0.05}}$. 
        \item For $j\neq y$, $z^{r}_{j}\in\{0\}\cup [\Omega(\rho_r), \rho_r]$ satisfying $\Pr(z^{r}_{j}\in [\Omega(\rho_r), \rho_r]) =\frac{s}{K}$, where $s<K$ (we choose $s=K^{0.1}$) to control feature sparsity and $\rho_r = \frac{1}{\polylog(K)}$.
    \end{itemize}
    \item With probability $1-\mu_r$, $z^r$ is generated from the   \textit{sufficient} class: 
    \begin{itemize}
        \item $z^{r}_{y}\in [1, C_r]$, where $C_r>1$ is a constant. 
        \item  For $j\neq y$, $z^{r}_{j}\in\{0\}\cup [\Omega(1),c_r]$ satisfying $\Pr(z^{r}_{j}\in [\Omega(1),c_r]) =\frac{s}{K}$
        ,where $c_r$ is a constant $<\frac{1}{2}$.
    \end{itemize}

\end{itemize}
In our settings, $\|z^{r}\|_0=\Theta(s)$ is a sparse vector. Each class $j$ has its associated feature $\Mb^{r}_j$ in each modality $\mathcal{M}_r$. We observe that for the sufficient class, the value of true label's coordinate in $z^{r}$, i.e.,  $z^{r}_{y}$, is more significant than others. On the other hand, for the insufficient class, the target coordinate is smaller than the off-target signal in terms of order.

\paragraph{Significance of the \textit{insufficient} class.}
In practice, different modalities are of various importance under specific circumstance~\cite{ngiam2011multimodal,liu2018learn,gat2020removing}. It is common that information from one single modality may be incomplete to build a good classifier~\cite{yang2015auxiliary,liu2018learn, gupta2016cross}. The restrictions on $z^{r}_{y}$ well capture this property, in the sense that there is a non-trivial probability $\mu_r$ that 
the coefficient $z^{r}_{y}$ is relatively small and easy to be concealed by the off-target signal. Therefore, when $z^{r}$ 
falls into this category, it provides \textit{insufficient} information for the classification task. Given modality $\mathcal{M}_r$, we call  $\Xb^{r}$  insufficient data if $z^{r}$ comes from the insufficient class, otherwise sufficient data.  Our data model distinguishes the multi-modal learning from previous well-studied multi-view analysis, which assumes that each view  is sufficient for classification~\cite{sridharan2008information}. Our classification is motivated by the  distribution studied in~\cite{allen2020towards}, where they utilize different levels of feature's coefficient to model the missing of certain features.

\paragraph{Noise Model:} We allow the input to incorporate a general Gaussian noise plus feature noise, i.e., 
$$\xi^{r}={\xi^{r}}^{\prime}+\mathbf{M}^{r} \alpha^{r}$$
Here, the Gaussian noise ${\xi^{r}}^{\prime} \sim \mathcal{N}\left(0, \sigma_g^{2} \mathbf{I}_{d^r}\right)$. 
The spike noise $\alpha^{r}$ is any coordinate-wise independent
non-negative random variable satisfying $\alpha_{y}^{r}=0$ and  $\|\alpha^{r}\|_{\infty}\leq \alpha$  , where $\alpha>0$ is the strength of the feature noise. We consider $\alpha = \frac{1}{K^{0.6}}$. 

Finally, we use $\mathcal{P}$ to denote the final data distribution of $(\Xb,y)$, and the marginal distribution of $(\Xb^{r},y)$ is denoted by $\mathcal{P}^{r}$.


\subsection{Learner Network}
\label{sec2.2}

We present the learner networks for both multi-modal learning and uni-modal learning. 
To start, we first define a smoothed version of ReLU activation function.
\begin{definition} The smoothed ReLU function is defined as
    $$
    \sigma(x) \stackrel{\text { def }}{=} \begin{cases}0 &  x \leq 0 ; \\ x^q/(\beta^{q-1}q) & x \in[0, \beta] ; \\ x-\beta\left(1-\frac{1}{q}\right)  &  x \geq \beta\end{cases}
    $$
    where  $q \geq 3$ is an integer and $\beta=\frac{1}{\polylog(K)}$.
\end{definition}
Such activation function is utilized as a proxy to study the behavior of neural networks with ReLU activation in prior theoretical analysis~\cite{allen2020towards, li2018algorithmic, haochen2021shape, woodworth2020kernel}, since it exhibits similar behaviour to the ReLU activation in the sense that $\sigma(\cdot)$ is linear when $x$ is large and becomes smaller when $x$ approaches zero. Moreover, it has desired property that the gradient of $\sigma(\cdot)$ is continuous. 
Besides, empirical studies illustrate that neural networks with polynomial activation have a matching performance compared to ReLU activation~\cite{allen2020backward}.

\begin{figure}[!t]
\vskip 0.2in
    \centering
    \includegraphics[width=.6\linewidth]{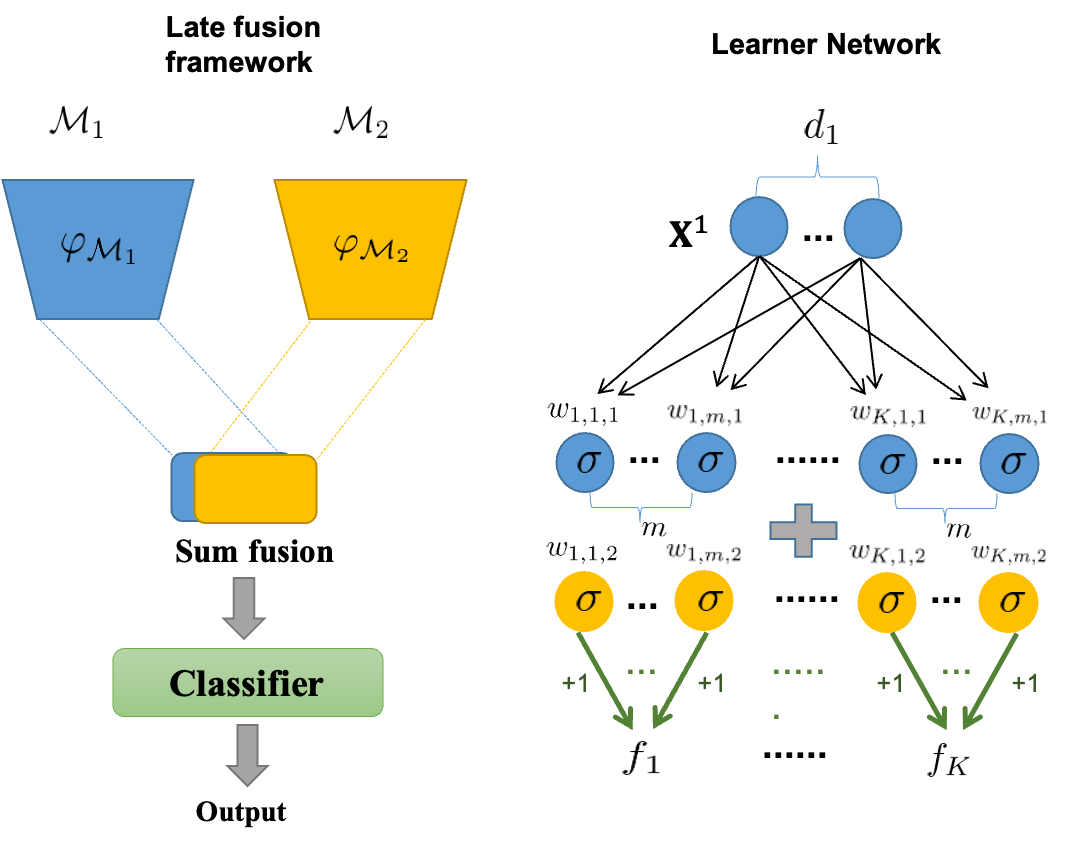}
    \vskip -0.02in
    \caption{Late fusion framework and our learner network.}
    \label{fig-learner}
\end{figure}
\paragraph{Multi-modal network:} We consider a late-fusion~\cite{wang2020makes} model on two modalities $\mathcal{M}_1$, and $\mathcal{M}_2$, which is illustrated by the left of Figure~\ref{fig-learner}.
Each modality is processed
by a single-layer neural net $\varphi_{\mathcal{M}_{r}}: \mathbb{R}^{d_r}\rightarrow \mathbb{R}^{M}$ with smoothed ReLU activation $\sigma(\cdot)$, where 
$M$ is the number of neurons.  
Then their features are fused by sum operation and passed to  a single-layer linear classifier $\mathcal{C}:\mathbb{R}^{M}\rightarrow \mathbb{R}^{K}$ 
 to learn the target. We  consider 
  $M=K\cdot m$ with $m=\polylog(K)$. 
More precisely, as illustrated in Figure~\ref{fig-learner}, the multi-modal network is formulated as follows:
\begin{align}
    &f(\Xb)=\left(f_{1}(\Xb), \ldots, f_{K}(\Xb)\right) \in \mathbb{R}^{K}\nonumber,\\
    &f_j(\Xb)= \sum^{m}_{l=1} \sigma(\la w_{j,l,1}, \Xb^{1}\ra)+\sigma(\la w_{j,l,2}, \Xb^{2}\ra)
    \end{align}
 where $w_{j,l,r}\in\mathbb{R}^{d_r}$ is the $(j-1)\cdot m+l$-th neuron of $\varphi_{m_{r}}$. Denote $\Wb^{r}$ the collection of weights $w_{j,l,r}$ and $\Wb^{r}_{j}:=(w_{j,1,r},\cdots,w_{j,m,r})^{\top}\in \mathbb{R}^{m\times d_r} $. Then the modality encoder of $\mathcal{M}_r$ can be written as:
 $$
 \varphi_{\mathcal{M}_{r}}(\Wb^{r},\Xb^{r})=\left( \sigma( {\Wb^{r}_1}^{\top}\Xb^{r}),\cdots\sigma( {\Wb^{r}_{K}}^{\top}\Xb^{r}) \right)
 $$
 where $\sigma(\cdot)$ is applied element-wise. 
 The classifier layer simply connects the entries from $(j-1)\cdot m+1$-th to $j\cdot m$-th   to the $j$-th output $f_{j}(\cdot)$ with non-trainable weights all equal to $1$.  The assumption that the second layer is fixed is common in previous
 works~\cite{du2018gradient,ji2019polylogarithmic,sarussi2021towards}. Moreover, theoretical analysis in \cite{huang2021makes} indicates that the success of multi-modal learning relies essentially on the learning of the hidden encoder layer. Nevertheless, we emphasize that our theory can easily adapt to the case where the second
 layer is trained.

 \paragraph{Uni-modal network:} The network architecture of uni-modal is similar except that the fusion step is omitted. Mathematically, $f^{\text{uni},r}:\mathbb{R}^{d_r}\rightarrow \mathbb{R}^{K}$ is defined as follows:
 \begin{align}
     &f^{\text{uni},r}(\Xb^{r})=\left(f^{\text{uni},r}_{1}(\Xb^{r}), \ldots, f^{\text{uni},r}_{K}(\Xb^{r})\right) \in \mathbb{R}^{K}\nonumber,\\
     &f^{\text{uni},r}_j(\Xb^{r})= \sum^{m}_{l=1} \sigma(\la \nu_{j,l,r}, \Xb^{r}\ra)
 \end{align}
 where $\nu_{j,l,r}\in\mathbb{R}^{d_r}$ denotes the weight. We use $\varphi^{\text{uni}}_{\mathcal{M}_r}$ to denote the modality encoder in uni-modal network.
  \paragraph{Training data:} We are given $n$  multi-modal data pairs $\{\Xb_i,y_i\}^{n}_{i=1}$ sampled from $ \mathcal{P}$, denoted by $\mathcal{D}$. We use $\mathcal{D}^{r}$ to denote the uni-modal data pairs $\{\Xb^{r}_i,y_i\}^{n}_{i=1}$ from $\mathcal{M}_r$. Moreover, we use $\mathcal{D}_{s}$ to denote the data pair that both $\Xb^{1}$ and $\Xb^{2}$ are sufficient data, and $\mathcal{D}_{i}$ to denote the data 
  that at least one  modality is insufficient. Denote the number of sufficient and insufficient data respectively as $n_s$ and $n_i$.
  
 \paragraph{Training Algorithm:}  We consider to learn the model parameter $\Wb$($\Wb^r$) by optimizing the empirical cross-entropy loss using gradient descent with learning rate $\eta>0$, which is a popular training combination investigated in the literature, e.g., ~\cite{wang2020makes, simonyan2014two}. 
 
 \begin{itemize}
     \item For multi-modal, the empirical loss is 
     \begin{align}
         \mathcal{L}(f)=\frac{1}{n}\sum_{(\Xb,y)\in \mathcal{D}} \mathcal{L}(f;\Xb,y)
     \end{align}
     where $\mathcal{L}(f;\Xb,y)= -\log \frac{\exp(f_{y}(\Xb))}{\sum_{j\in [K]}\exp(f_{j}(\Xb))}$. We initialize $w^{(0)}_{j,l,r}\sim\mathcal{N}(0,\sigma^2_0\mathbf{I}_{d_r})$ where $\sigma_0=\frac{1}{\sqrt{K}}$.\footnote{Such initialization is standard in practice.} We use $f^{(t)}$ to denote the multi-modal network with $f$ with weights $\Wb^{(t)}$ at iteration $t$. The gradient descent
     update rule is:
     $$w_{j,l,r}^{(t+1)}=w_{j,l, r}^{(t)}-\eta \cdot \nabla_{w_{j,l,r}} \mathcal{L}(f^{(t)})$$
     \item Similarly, for uni-modal, the empirical loss and gradient update rule is defined as follows:
     \begin{align}
         \mathcal{L}(f^{\text{uni},r})=\frac{1}{n}\sum_{(\Xb,y)\in \mathcal{D}^{r}} \mathcal{L}(f^{\text{uni},r};\Xb^{r},y)\\
        \nu_{j,l,r}^{(t+1)}=\nu_{j,l, r}^{(t)}-\eta \cdot \nabla_{\nu_{j,l,r}} \mathcal{L}({f^{\text{uni},r}}^{(t)})
     \end{align}
     where $\mathcal{L}(f^{\text{uni},r};\Xb^{r},y)= -\log \frac{\exp(f^{\text{uni},r}_{y}(\Xb^r))}{\sum_{j\in [K]}\exp(f^{\text{uni},r}_{j}(\Xb^r))}$, and $\nu^{(0)}_{j,l,r}\sim\mathcal{N}(0,\sigma^2_0\mathbf{I}_{d_r})$.
 \end{itemize}
\section{Main Results}\label{sec4}
We present the main theorems of the paper here. We start with the optimization and generalization guarantees of the uni-modal network. Then, we study the feature learning process of multi-modal networks with joint training. We show that for naive joint training, each modality's encoder has a non-trivial probability to learn unfavorable feature representations. Combining with the special structure of insufficient data, we immediately establish the performance gap between the best uni and multi-modal theoretically. 

\subsection{Uni-modal Network Results}
The following theorem states that after enough iterations, the uni-modal networks can attain the global minimum of the empirical training loss, and such uni-modal solution also has a good test performance.
\begin{theorem}\label{thm-uni}
    For every $r\in[2]$, for sufficiently large $K>0$ and every $\eta \leq \frac{1}{\operatorname{poly}(K)}$, after $T=\frac{\text { poly }(K)}{\eta}$ many iteration, the learned uni-modal network ${f^{\text{uni},r}}^{(t)}$ \textit{w.h.p} satisfies:
       \begin{itemize}
           \item Training error is zero: 
           \begin{align*}
            \frac{1}{n}\sum_{(\Xb^{r},y)\in\mathcal{D}^{r} }\mathbb{I}\left\{\exists j\neq y: \right.{f_{y}^{\text{uni},r}}^{(T)}(\Xb^{r})  \leq   \left. {f_{j}^{\text{uni},r}}^{(T)}(\Xb^r)\right\}=0.
           \end{align*}
           \item  The test error satisfies:
            \begin{align*}
            \Pr_{(\Xb^{r},y)\sim\mathcal{P}^{r}}(\exists j\neq y: {f_{y}^{\text{uni},r}}^{(T)}(\Xb^r)\leq {f_{j}^{\text{uni},r}}^{(T)}(\Xb^r) )
        =(1 \pm o(1))\mu_r
           \end{align*}
            
   
       
       \end{itemize}

              \end{theorem}
   
 Recall that $\mu_r$ represents the proportion of data falling into the insufficient class for modality $\mathcal{M}_r$. Note that ${f^{\text{uni},r}}^{(T)}$ not only minimizes the training error, but the primary source of its test error is from the insufficient data that cannot provide enough feature-related information for the classification task. Therefore, Theorem~\ref{thm-uni} suggests that the uni-modal networks ${f^{\text{uni},r}}$ can learn ideal feature representations for the used single modality $\mathcal{M}_r$. 

\subsection{Multi-modal Network with Joint Training}

In order to evaluate how good the feature representation learned by the encoder of each modality in joint training, we consider a uni-modal network ${f^{r}}^{(t)}:=\mathcal{C}(\varphi^{(t)}_{\mathcal{M}_r})$, where  $\varphi^{(t)}_{\mathcal{M}_r}$ is the $\mathcal{M}_r$'s encoder learned by joint training at iteration $t$, and $\mathcal{C}$ is the non-trainable linear head we defined in Section~\ref{sec2.2}. The input for ${f^{r}}^{(t)}$ is simply  the data $\Xb^{r}$ from $\mathcal{M}_r$. We will measure the goodness of $\varphi^{(T)}_{\mathcal{M}_r}$ by the test performance of ${f^{r}}^{(T)}$, which is analogous to the method widely employed in empirical studies of self-supervised learning to evaluate the learned feature representations~\cite{chen2020simple}.
    

\begin{theorem}\label{thm-mul}
 For sufficiently large $K>0$ and every $\eta \leq \frac{1}{\operatorname{poly}(K)}$, after $T=\frac{\text { poly }(K)}{\eta}$ many iteration, for the multi-modal network $f^{(t)}$, ${f^{r}}^{(t)}:=\mathcal{C}(\varphi^{(t)}_{\mathcal{M}_r})$ \textit{w.h.p} :
    \begin{itemize}
        \item Training error is zero: 
        $$
\frac{1}{n}\sum_{(\Xb,y)\in\mathcal{D} }\mathbb{I}\{\exists j\neq y: f_{y}^{(T)}(\Xb) \leq f_{j}^{(T)}(\Xb)\}=0.
        $$
        \item For $r\in[2]$, with probability $p_{3-r}>0$, the test error of ${f^{r}}^{(T)}$ is high:
        \begin{align*}
            \Pr_{(\Xb^{r},y)\sim\mathcal{P}^{r}}(\exists j\neq y: {f_{y}^{r}}^{(T)}(\Xb^r)\leq {f_{j}^{r}}^{(T)}(\Xb^{r}) )\geq \frac{1}{K}
        \end{align*}
         where $p_1+p_2=1-o(1)$, and $p_r\geq m^{-O(1)}$, $\forall r\in[2]$.
         

    
    \end{itemize}

           \end{theorem}
\paragraph{Discussion of $p_r$:} $p_r$ represents the probability that modality $\mathcal{M}_{3-r}$ fails to learn a good feature representation. The specific values of $p_1$ and $p_2$ are associated with the relative relation between the marginal distribution of $z^{1}$ and $z^{2}$ from \textit{sufficient} class. Typically, if the lower bound of $\su(z^{r}_y)$ 
is larger than the upper bound of $\su(z^{3-r}_y)$, $p_{r}$ tends to be larger than $p_{3-r}$.  Nevertheless, our results indicate that no matter how such relation varies, even in extreme cases (e.g., the lower bound of $\su(z^{r}_y)$ is excessively larger than the upper bound of $\su(z^{3-r}_y)$, both of $p_1$ and $p_2$ are lower bounded by a non-trivial value.

\paragraph{Feature representations learned in joint training are unsatisfactory.} From the optimization perspective, Theorem~\ref{thm-mul} shows that the multi-modal networks with joint training can be guaranteed to find a point that achieves zero error on the training set. However, such a solution is not optimal for both modalities. In particular, the output of the uni-modal 
network ${f^{r}}^{(T)}$, which we defined earlier to assess the quality of the learned modality encoder for $\mathcal{M}_r$, has a non-negligible probability to generalize badly and give a test error over $1/K$ (almost random guessing for $K$-classification, and exceedingly larger than ${f_{y}^{\text{uni},r}}^{(T)}$). The occurrence of such poor test performance indicates that \textit{w.h.p}, at least one of the modality encoding networks learned relatively deficient knowledge about the modality-associated features.

\paragraph{Remark.} Originally, the intention of joint training is that for a multi-modal sample, if some of these modalities have \textit{insufficient} structure,  the information provided by remaining  \textit{sufficient} modalities can assist training and improve the accuracy. Nevertheless, Theorem~\ref{thm-mul} indicates that adding more modalities through naive joint possibly impairs the feature representation learning of the original modalities
Consequently, the modal not only fails to exploit the extra modalities, but also loses the expertise of the original modality. 

Based on the results in Theorem~\ref{thm-mul}, we are able to characterize the performance gap between uni-modal and multi-modal with joint training in the following corollary.

\begin{corollary}[Failure of Joint Training]\label{col}
    Suppose the assumptions in Theorem~\ref{thm-mul} holds, \textit{w.h.p}, for joint training,  the learned multi-modal network $f^{(T)}$ satisfies:
    \begin{align*}
        \Pr_{(\Xb,y)\sim\mathcal{P}}(\exists j\neq y:f_{y}^{(T)}(\Xb)\leq
        f_{j}^{(T)}(\Xb) )\in [\sum_{r\in[2]} (p_r-o(1))\mu_r, \sum_{r\in[2]} (p_r+o(1))\mu_r]
    \end{align*}
Combining with the results in Theorem~\ref{thm-uni}, we immediately obtain:
\begin{align*}
    \Pr_{(\Xb,y)\sim\mathcal{P}}(\exists j\neq y:&f_{y}^{(T)}(\Xb)\leq  
    f_{j}^{(T)}(\Xb) )\geq
    \min_{r\in[2]} \Pr_{(\Xb^{r},y)\sim\mathcal{P}^{r}}(\exists j\neq y:& {f_{y}^{\text{uni},r}}^{(T)}(\Xb^r)\leq {f_{j}^{\text{uni},r}}^{(T)}(\Xb^r) )
\end{align*}
\end{corollary}
Notice that the test error of the joint training is approximately the weighted average of the test error of uni-modal network and is affected by two sets of factors $\{p_r\}_{r\in[2]}$, $\{\mu_{r}\}_{r\in[2]}$. The corollary has simple intuitive implications. If there exists a “strong” modality with a smaller $\mu_r$ (less insufficient structure) and a larger $p_r$ (more likely to prevail during training), the closer the joint training is to the best uni-modal, since the other modality is too weak to interfere the feature learning process of the strong modality.

\section{Proof Outline}\label{sec5}
In this section we provide the proof sketch of our theoretical results. We provide overviews of multi-modal and uni-modal training process in Section~\ref{sec5.1} and~\ref{sec5.2} respectively, to provide intuitions for our proof. The complete proof is deferred to the supplementary.

\subsection{Overview of the Joint Training Process}\label{sec5.1}
Given modality $\mathcal{M}_r$ and class $j\in[K]$, we characterize the feature learning of its modality encoder $\varphi_{\mathcal{M}_r}$ in the training process by quantity: 
$\Gamma^{(t)}_{j,r}=\max_{l\in[m]}[\la \Mb^{r}_{j},w^{(t)}_{j,l,r} \ra ]^{+}. \label{eq:Gamma}$
 It can be seen that a larger $\Gamma^{(t)}_{j,r}$ implies better grasp of the target feature $ \Mb^{r}_{j}$.

We will show that the training dynamics of multi-modal joint training can be decomposed into two phases:  1) Some special patterns of the neurons in the learner networks emerge and become singletons due to the random initialization, which demonstrates the phenomenon of \textit{modality competition}; 2) As long as the neurons are activated by the winning modality, they will indeed converge to such modality, and ignore the other. 

\paragraph{Phase 1: modality competition from random initialization.} 
Our proof begins by showing how the neurons in each modality encoder $\varphi_{\mathcal{M}_r}$ are emerged from random initialization. In particular, we will show that, despite the existence of multiple class-associated features (comes from different modalities),  only one of them will be quickly learned by its corresponding encoding network, while the others will barely be discovered out of the random initialization. We call this phenomenon “modality competition” near random initialization, which demonstrates the origin of the sub-optimality of naive joint training.


Recall that at iteration $t=0$, the  weights are initialized as $w_{j,l,r}^{(0)} \sim \mathcal{N}\left(0, \sigma_{0}^{2} \mathbf{I}_{d_r}\right)$. For $j\in[K]$, $r\in[2]$, define the following data-dependent parameter: 
\begin{align*}
        d_{j, r}(\mathcal{D})=\frac{1}{n\beta^{q-1}}\sum_ {(\Xb,y)\in\mathcal{D}_{s}}\mathbb{I}\{y=j\}  \left(z^{r}_{j}\right)^{q}
    \end{align*}
   Recall that $\mathcal{D}_{s}$ denotes the data pair that both $\Xb^{1}$ and $\Xb^{2}$ are sufficient data, i.e., the sparse vectors $z^{1}$ and $z^{2}$ both come from the \textit{sufficient} class.
   Therefore, $d_{j, r}(\mathcal{D})$ represents the strength of the target signal for sufficient data from class $j$ and modality $\mathcal{M}$. Applying standard properties of the Gaussian distribution, we show the following critical property: 
\begin{property}~\label{prop}
    For each class $j\in [K]$, \textit{w.h.p}, 
    there exists $r_j\in [2]$, s.t.
    \begin{align*}
        \Gamma^{(0)}_{j,r_j}[d_{j, r_j}(\mathcal{D})]^{\frac{1}{q-2}} \geq \Gamma^{(0)}_{j,3-r_j} [d_{j, 3-r_j}(\mathcal{D})]^{\frac{1}{q-2}}\cdot(1+\frac{1}{\operatorname{polylog}(K)})
    \end{align*}
\end{property}
In other words, by the property of random
Gaussian initialization, for each class $j\in [K]$,  there will be a $\mathcal{M}_{r_{j}}$, termed as  winning modality,  where the maximum correlation between $\Mb^{r_{j}}_{j}$ and one of the neurons 
of its corresponding encoder $\varphi_{\mathcal{M}_{r_{j}}}$ is slightly higher than the other modality $\mathcal{M}_{3-r_{j}}$. In our proof, we will identify the following phenomenon during the training: 
\begin{Msg}{Modality Competition}
For every $j\in[K]$, at every iteration $t$, if $\mathcal{M}_{r_{j}}$  is the winning modality, then $\Gamma^{(t)}_{j,r_{j}}$ will grow faster than $\Gamma^{(t)}_{j,3-r_{j}}$. When $\Gamma^{(t)}_{j,r_{j}}$ reaches the threshold $\Theta(\beta)=\widetilde{\Theta}(1)$, $\Gamma^{(t)}_{j,3-r_{j}}$ still stucks at initial level around $\widetilde{O}(\sigma_0)$.
    \end{Msg}
    
\subparagraph{Probability of winning.} Observing that $ d_{j, r}(\mathcal{D})$ is related to the marginal distribution of $z^{r}$, we will prove that even in the extreme setting that $z_{j}^{r}\gg z_{j}^{3-r}$ for $j=y$ almost surely, which implies $ d_{j, r}(\mathcal{D})\gg d_{j, 3-r}(\mathcal{D})$ with high probability, $\mathcal{M}_{3-r}$ has a slightly notable probability, denoted by $p_{j,3-r}\geq m^{-O(1)}$, to be the winning modality for class $j$ out of random initialization. Noticing that $p_{j,r}$ also represents the probability that the modality $\mathcal{M}_{3-r}$ fails to be discovered for class $j\in [K]$ at the beginning, our subsequent analysis will illustrate that such a lag situation will continue, leading to bad feature representations for $\mathcal{M}_{3-r}$ with probability $p_r=\sum_{j\in[K]}p_{j,r}/K\geq m^{-O(1)}$.

\subparagraph{Intuition:} Technically, in this phase, the activation function $\sigma(\cdot)$  is still in the polynomial or negative regime, and we can reduce the dynamic to tensor power method~\cite{anandkumar2015analyzing}. We observe that the update of $\Gamma^{(t+1)}_{j,r}$ is approximately:
 $\Gamma^{(t+1)}_{j,r}\approx \Gamma^{(t)}_{j,r}+ \eta\cdot A^{(t)}_r (\Gamma^{(t)}_{j,r})^{q-1}$, 
with $A^{(t)}_r =\Theta(1)$, which is similar to power method for $q$-th ($q\geq 3$) order tensor decomposition. By the behavior observed in randomly initialized tensor power method~\cite{anandkumar2015analyzing,allen2020towards},  
a slight initial difference can create very dramatic growth gap. Based on this intuition, we introduce the Property~\ref{prop} to characterize how much difference of initialization can make one of the modalities stand out to be the winning modality and propose the modality competition to further show that the neurons for the winning modality maintain the edge until they become roughly equal to $\Theta(\beta)=\widetilde{O}(1)$, while the others are still around initialization $\widetilde{O}(\sigma_0)$ (recall that the networks are initialized by $\mathcal{N}(0,\sigma_0\mathbf{I}_{d_r})$).
    
\subparagraph{Remark.}  The idea that only part of modalities will win during the training is also motivated by a phenomenon called ``winning the lottery ticket"  identified in recent theoretical analysis for over-parameterized neural networks~\cite{li2020learning, wen2021toward, allen2020feature}. That is,  for over-parameterized neural networks, only a small fraction of neurons has much larger norms than an average norm. Their works focus on who wins in the neural networks, while our focus is the winner of inputs, the modality.    
    
\paragraph{Phase 2: converge to the winning modality.}
The next phase of our analysis begins when one of the modalities already won the competition near random initialization, and focuses on showing that it will dominate until the end of the training. 
After the first phase,  the pre-activation of the winning modality's neurons will reach the linear region, while the pre-activation of the others still remain in the polynomial region or even negative. Yet, the loss starts to decrease significantly, and we prove that $\Gamma^{(t)}_{j,3-r_{j}}$ will no longer exceed $\widetilde{O}(\sigma_0)$ until the training loss are close to converge.
Therefore, the winning modality will remain the victory throughout the training.



\subsection{Overview of the Uni-modal Training Process}\label{sec5.2}
The training process of uni-modal can also be decomposed into two phases, i.e., 1) learning the pattern, and  2)  converging to the learned features. Similarly, we define $\Psi^{(t)}_{j,r}=\max_{l\in[m]}[\la  \Mb^{r}_{j},\nu^{(t)}_{j,l,r} \ra ]^{+}$ to quantify the feature learning for the  uni-modal network $f^{\text{uni},r}$.

We briefly describe the difference between the uni-modal and the joint-training case. The main distinction arises from Phase 1. Intuitively, since there is only one predictive signal source without competitors, we  prove that the network will \textbf{focus on} learning the features from the given modality in the first phase. In particular,  $\Psi^{(t)}_{j,r}$ will grow fast to $\widetilde{O}(1)$ at the end of this phase. Then in Phase 2, the uni-modal will continue to explore the the learned patterns until the end of training.



\section{Conclusions}\label{sec6}
In this paper, we provide novel theoretical understanding towards a qualitative phenomenon commonly observed in deep multi-modal applications, that the best uni-modal network outperforms the multi-modal network trained jointly under late-fusion settings. We analyze the optimization process and theoretically establish the performance gaps for these two approaches in terms of test error. 
In theory, we characterize the modality competition  phenomenon to tentatively explain the main cause of the sub-optimality of joint training. Empirical results are provided to verify that our theoretical framework does coincide with the superior of the best uni-modal networks over joint training in practice. 
To a certain extent, our work reflects how the prevailing pre-training methods~\cite{m6}, which are capable of extracting favorable features for every modality,  lead to better performance for multi-modal learning. 
Our results also facilitate further theoretical analyses in multi-modal learning through a new mechanism that focuses on how modality encoder learns the features. 
%


\bibliography{example_paper}
\bibliographystyle{plainnat}

\newpage
\appendix
\onecolumn

\section{Proofs for Multi-modal Joint Training}
In this section, we will provide the proofs of Theorem~\ref{thm-mul} for multi-modal joint training.  We will first focus on some properties and characterizations for modality at initialization. Our analysis actually rely on an induction hypothesis. Then we will introduce the hypothesis and prove that it holds in the whole training process. Finally, we will use this hypothesis to complete the proof of our main theorem.
\subsection{Notations and Preliminaries}~\label{sec-not}
We first describe some preliminaries before diving
 into the proof.

\paragraph{Global Assumptions.} Throughout the proof in this section,
\begin{itemize}
    \item We choose $\sigma_0^{q-2}=\frac{1}{K}$ for $q\geq 3$, where $\sigma_0$ controls the initialization magnitude.
    \item $m=\polylog(K)$, where $m$ controls the number of neurons. 
    \item $\sigma_g=O(\sigma_0^{q-1})$, wehre $\sigma_g$ gives the magnitude of gaussian noise.
    \item $\alpha=\widetilde{O}(\sigma_0)$, where $\alpha$ controls the feature noise.
    \item $\frac{s}{K}\leq \widetilde{O}(\sigma_0)$, where $s$ controls the feature sparsity.
    \item $n_i\leq \frac{K^2\gamma^{q-1}}{s}$, where $n_i$ is the size of the insufficient multi-modal training data.
    \item $\rho_r=\frac{1}{\operatorname{poly} \log (K)}$ where $\rho_r$ control the off-target signal for insufficient data.
    \item $n \geq \widetilde{\omega}\left(\frac{K}{\sigma_{0}^{q-1}}\right), n \geq \widetilde{\omega}\left(\frac{k^{4}}{s^{2} \sigma_{0}}\right), \frac{T}{\eta \sqrt{d_r}} \leq 1 / \operatorname{poly}(K)$ for $r\in[2]$.
    \item $\gamma_r^{q-1}\leq \frac{1}{K}$ for $r\in[2]$, where $\gamma_r$ controls the target signal for insufficient data.
\end{itemize}
\paragraph{Network Gradient.} Given data point $(\Xb, y) \in \mathcal{D}$, in every iteration $t$ for every $j \in[K]$, $l \in[m]$, $r \in [2]$
\begin{align*}
&-\nabla_{w_{j, l, r}} \cL( f ; \Xb, y)=\left(\mathbb{I}\{j=y\}-\ell_{j}(f, \Xb)\right)  \sigma^{\prime}\left(\left\langle w_{j, l, r}, \Xb^{r}\right\rangle\right) \Xb^{r}
\end{align*}
where $\ell_j(f,\Xb):=\frac{\exp(f_j(\Xb))}{\sum_{i\in [K]}\exp(f_i(\Xb))}$, $\mathbb{I}\{\cdot\}$ is the indicator, and $\sigma^{\prime}(\cdot)$ denotes the derivative of the smoothed ReLU function.
\paragraph{Gaussian Facts.}

\begin{lemma}\label{lemma-prob}
    Consider two  Gussian random vector $(X_{1}, \ldots, X_{p})$ , $(Z_1,\cdots, Z_p)$, where $X_i \overset{\text{i.i.d.}}{\sim} \mathcal{N}\left(0, 1\right)$, $Z_i \overset{\text{i.i.d.}}{\sim} \mathcal{N}\left(0, \bar{\sigma}^2\right):$ 
\begin{itemize}
    \item[(a).] For $\bar{\sigma}\leq1$, for every $\epsilon>0$, with \textbf{at most} probability $O(\frac{1}{\poly (p)}+ \epsilon \log p)$: $\max_{i\in [p]} X_i =\max_{i\in [p]} Z_i\cdot (1\pm O(\epsilon))$ 
    \item[(b). ] For $\bar{\sigma}\geq1$, for every $\epsilon>0$, with \textbf{at least} probability $p^{-(\bar{\sigma}^2-1)}\cdot \Omega(\frac{1}{\bar{\sigma}})$: $\max_{i\in [p]} X_i \geq \max_{i\in [p]} Z_i$
\end{itemize}
\end{lemma}
\begin{proof}[Proof]
The lemma can be derived by anti-concentration theorems~\cite{chernozhukov2015comparison} and maximum Gaussian property~\cite{kamath2015bounds} using the standard Gaussian analysis. The proof follows from Proposition B.2 in ~\cite{allen2020towards}, and here we omit the proof details. 

\end{proof}

\subsection{Modality Characterization at Initialization}
Define the following data-dependent parameter: 
\begin{align*}
        d_{j, r}(\mathcal{D})=\frac{1}{n\beta^{q-1}}\sum_ {(\Xb,y)\in\mathcal{D}_{s}}\mathbb{I}\{y=j\}  \left(z^{r}_{j}\right)^{q}
    \end{align*}
    Recall $\mathcal{D}_{s}$ denotes the data pair whose sparse vectors $z^{1}$ and $z^{2}$ both come from sufficient class.

    For each class $j\in [K]$, let us denote:
    $$
    \Gamma_{j, r}^{(t)} \stackrel{\text { def }}{=} \max _{l \in[m]}\left[\left\langle w_{j, l, r}^{(t)}, \Mb_{j}^{r}\right\rangle\right]^{+}\quad \text { and } \quad \Gamma_{j}^{(t)} \stackrel{\text { def }}{=} \max _{r \in[2]}\Gamma_{j, r}^{(t)}
    $$

Let us give the following definitions and results to characterize each modlaity's property at initialization:
\begin{definition}[Winning Modality] For each class $j\in[K]$, at iteration $t=0$, if there exists $r_j\in[2]$, s.t. 
\begin{align*}
    \Gamma^{(0)}_{j,r_{j}}d_{j, r_{j}}(\mathcal{D})^{\frac{1}{q-2}} &\geq \Gamma^{(0)}_{j,3-r_{j}} d_{j, 3-r_{j}}(\mathcal{D})^{\frac{1}{q-2}}\cdot(1+\frac{1}{\operatorname{polylog}(K)})
\end{align*}
then we refer the modality $\mathcal{M}_{r_j}$ as the winning modality for class $j$. It is obvious that \textbf{at most} one of modalities can win. 
\end{definition}
\begin{lemma}[\textbf{Wining Modality Characterization}]\label{lemma-win}
For every $j\in[K]$, denote the probability that modality $\mathcal{M}_r$ is the winning modality as $p_{j,r}$, then we have 
\begin{itemize}
    \item $p_{j,1}+p_{j,2}\geq 1-o(1)$.
    \item $p_{j,r}\geq(\frac{1}{\polylog(K)})^{O(1)}$ for every $r\in[2]$. 
\end{itemize}
\end{lemma}
\begin{proof}[Proof of Lemma~\ref{lemma-win}]

For the first argument, if neither of modalities wins, then we must have:
    $$
\Gamma_{j, r}^{(0)}=\Gamma_{j, 3-r}^{(0)}\left(\frac{d_{j, 3-r}(\mathcal{D})}{d_{j, r}(\mathcal{D})}\right)^{\frac{1}{q-2}}\left(1 \pm O\left(\frac{1}{\polylog(K)}\right)\right)
$$
By our assumption, we have $\frac{d_{j, 3-r}(\mathcal{D})}{d_{j, r}(\mathcal{D})}\leq 1$ and is fixed given the training data. Letting $p=m$, $\epsilon=\frac{1}{m\log m}$, applying Lemma~\ref{lemma-prob} $(a)$, we obtain the probability that this event occurs is at most $O(\frac{1}{\polylog K})
$ (Recall that $m=\polylog(K)$).

For the second argument, we just need to prove that $\Gamma_{j, 3-r}^{(0)}\left(\frac{d_{j, 3-r}(\mathcal{D})}{d_{j, r}(\mathcal{D})}\right)^{\frac{1}{q-2}}$ has a non-trival probability to be larger than $\Gamma_{j, r}^{(0)}$. We can apply the conclusion of $(b)$ in Lemma~\ref{lemma-prob}, observing that $\bar{\sigma}=(\frac{d_{j, 3-r}(\mathcal{D})}{d_{j, r}(\mathcal{D})})^{\frac{1}{q-2}}$ is a constant and then obtain that
\begin{align*}
    \Pr(\Gamma_{j, 3-r}^{(0)}\left(\frac{d_{j, 3-r}(\mathcal{D})}{d_{j, r}(\mathcal{D})}\right)^{\frac{1}{q-2}}\leq \Gamma_{j, r}^{(0)})\geq \frac{1}{m^{O(1)}}=\frac{1}{\polylog(K)^{O(1)}}
\end{align*}
Hence, we compelets the proof.

\end{proof}

\subsection{Induction Hypothesis}
Given a data $\Xb$, define:
\begin{align*}
    \mathcal{S}^{r}(\Xb):=\{j\in [K]: \text{the $j$-th coordinate of $\Xb^{r}$'s sparse vector $z^{r}$ is not equal to zero, i.e. }  z^{r}_{j}\neq 0 \}
\end{align*}
We abbreviate $\mathcal{S}^{r}(\Xb)$ as $\mathcal{S}^{r}$ in our subsequent analyis for simplicity.

\begin{hypothesis}~\label{hypo}
\begin{enumerate}[label=\roman*]
    \item [] For sufficient data $(\Xb,y)\in\mathcal{D}_s$, for every $r\in[2]$, $l\in[m]$:
    \item for every $j=y$, or $j\in \mathcal{S}^{r}:$ $\left\langle w_{j,l,r}^{(t)}, \Xb^{r}\right\rangle=\left\langle w_{j,l,r}^{(t)}, \Mb^{r}_{j}\right\rangle z^{r}_{j} \pm \widetilde{o}\left(\sigma_{0}\right)$.
    \item else $\left|\left\langle w_{j, l,r}^{(t)}, \Xb^{r}\right\rangle\right| \leq \widetilde{O}\left(\sigma_{0} \right)$
       
   \item [] For insufficient data $(\Xb, y) \in \mathcal{D}_{i}$, every $l \in[m]$, every $r \in[2]$:
   
       \item  for every $j =y:$ $\left\langle w_{j, l,r}^{(t)}, \Xb^{r}\right\rangle=\left\langle w_{j, l, r}^{(t)}, \Mb_{j}^{r}\right\rangle z^{r}_{j}+\left\langle w_{j, l, r}^{(t)}, {\xi^{r}}^{\prime}\right\rangle \pm \widetilde{O}\left(\sigma_{0} \alpha K\right)$
       \item for every $j\in \mathcal{S}^{r}:$ $\left\langle w_{j,l,r}^{(t)}, \Xb^{r}\right\rangle=\left\langle w_{j,l,r}^{(t)}, \Mb^{r}_{j}\right\rangle z^{r}_{j} \pm \widetilde{o}\left(\sigma_{0}\right)$.
      \item for every $j =y$, if $\mathcal{M}_{3-r}$ is the winning modality for $j$, we have: $\left|\left\langle w_{j, l, r}^{(t)}, \Xb^{r}\right\rangle\right| \leq \widetilde{O}\left(\sigma_{0}\right)$
      \item else
    $\left|\left\langle w_{j, l, r}^{(t)}, \Xb^{r}\right\rangle\right| \leq \widetilde{O}\left(\sigma_{0} \right)$
  
   
   Moreover, we have for every $j \in[k]$,
       \item $\Gamma_{j}^{(t)} \geq \Omega\left(\sigma_{0}\right)$ and $\Gamma_{j}^{(t)} \leq \widetilde{O}(1)$.
       \item for every $l \in[m]$, every $r\in[2]$, it holds that $\left\langle w_{j, l, r}^{(t)}, \Mb_{j}^{r}\right\rangle \geq-\widetilde{O}\left(\sigma_{0}\right)$.
   \end{enumerate}
\end{hypothesis}
\paragraph{Proof overview of Induction Hypothesis~\ref{hypo}.} We will first characterize the training phases and then state some claims as consequences of statements of the hypothesis, which is crucial for our later proof. After that, we will analyze  the training process in every phases to prove the hypothesis.

Let us introduce some calculations assuming the hypothesis holds to simplify the subsequent proof. 

\begin{fact}[Function Approximation]~\label{fact-app}
    Let $Z_{j,r}(\Xb)=\mathbb{I}\{j=y\text{, or } j\in\mathcal{S}^{r}\}z_{j}^{r}$, $\Phi_{j, r}^{(t)} \stackrel{\text {def}}{=} \sum_{l \in[m]}\left[\left\langle w_{j, l, r}^{(t)}, \Mb_{j}^{r}\right\rangle\right]^{+}$ and $\Phi_{j}^{(t)} \stackrel{\text {def}}{=} \sum_{r \in[2]} \Phi_{j,r}^{(t)}$
for every $t$, every $(\Xb, y) \in \mathcal{D}_{s}$ and $j \in[K]$, or for every $(\Xb, y) \in \mathcal{D}_{i}$ and $j \in[K] \backslash\{y\}$,
$$
\begin{aligned}
f_{j}^{(t)}(X) &=\sum_{r \in[2]}\left(\Phi_{j, r}^{(t)} \times Z_{j, r}(\Xb)\right) \pm O\left(\frac{1}{\polylog(K)}\right)
\end{aligned}
$$
for every $(\Xb, y) \sim \mathcal{P}$, with probability at least $1-e^{-\Omega\left(\log ^{2} K\right)}$ it satisfies for every $j \in[K]$,
$$
f_{j}^{(t)}(X)=\sum_{r \in[2]}\left(\Phi_{j, r}^{(t)} \times Z_{j, r}(\Xb)\right) \pm O\left(\frac{1}{\polylog(K)}\right)
$$
Similarly, for $(\Xb^{r}, y)\sim\mathcal{P}^r$, for $r\in[2]$, \textit{w.h.p.} 
$$
{f^{r}_{j}}^{(t)}(\Xb)=\Phi_{j, r}^{(t)} \times Z_{j, r}(\Xb) \pm O\left(\frac{1}{\polylog(K)}\right)
$$
\end{fact}
 
\begin{fact}~\label{fact-err}
For every $(\Xb, y) \in \mathcal{D}$ and every $j \in[K]: \ell_{j}\left(f^{(t)}, \Xb\right)=O\left(\frac{e^{O\left(\Gamma_{j}^{(t)}\right) m}}{e^{O\left(\Gamma_{j}^{(t)}\right) m}+K}\right)$; Moreover, for every $(\Xb, y) \in \mathcal{D }_{i}$ and 
$j \in[K] \backslash\{y\}$, we have $\ell_{j}\left(f^{(t)}, \Xb\right)=O\left(\frac{1}{K}\right)\left(1-\ell_{y}\left(f^{(t)}, \Xb\right)\right)$
\end{fact} 

\begin{proof}
    $f^{(t)}_{j}(\Xb)=\sum_{l\in[m]}\sum_{r\in[2]} \sigma(\langle w_{j,l, r}^{(t)}, \Xb^{r} \rangle)$, by Induction Hypothesis~\ref{hypo},
    \begin{align}
       \sigma(\langle w_{j,l, r}^{(t)}, \Xb^{r} \rangle)\leq O(\frac{1}{m})+ [\langle w^{(t)}_{j,l, r}, \Mb^{r}_{j}\rangle]^{+} Z_{j,r}(\Xb)
    \end{align}
    Hence, $f^{(t)}_{j}(\Xb)\leq m\Gamma^{(t)}_{j}\cdot O(1)+O(1)$. Furthermore, for $(\Xb, y) \in \mathcal{D }_{i}$ and $j\neq y$, $\sum_{r\in[2]}Z_{j,r}(\Xb)\leq (\rho_1+\rho_2)$, then we have $f^{(t)}_{j}(\Xb)\leq m\Gamma^{(t)}_{j}\cdot (\rho_1+\rho_2)+O(1)=O(1)$.
\end{proof}

\subsection{Training Phase Characterization}
\begin{claim}~\label{claim-grow}
    Suppose Induction Hypothesis~\ref{hypo} holds, when $\Gamma_j^{(t)}=O\left(1 / m\right)$, then it satisfies
$$
\Gamma_{j}^{(t+1)}=\Gamma_{j}^{(t)}+\Theta\left(\frac{\eta}{K}\right) \sigma^{\prime}\left(\Gamma_{j}^{(t)}\right)  
$$
\end{claim}

\begin{proof}
    We consider the case that there exists $l,r$, s.t. $\langle w^{(t)}_{j,l, r}, \Mb^{r}_{j}\rangle$ reaches $\widetilde{\Omega}(\frac{1}{m})$. By gradient updates, we have:
    \begin{align*}
        \left\langle w_{j, l, r}^{(t+1)}, \Mb_{j}^{r}\right\rangle&\geq  \left\langle w_{j, l, r}^{(t)}, \Mb_{j}^{r}\right\rangle\\
    &+\frac{\eta}{n}\sum_{(\Xb,y)\in \mathcal{D}}\left[\mathbb{I}\{y=j\}\left(1-\ell_{j}\left(f^{(t)}, \Xb\right)\right)\left(\sigma^{\prime}(\langle w_{j,l, r}^{(t)}, \Xb^{r}\rangle)z^{r}_{j}-O\left(\sigma_{g}\right)\right)\right.\\
    &-\mathbb{I}\{y\neq j\}\left.\ell_{j}\left(f^{(t)}, \Xb\right)\left(\mathbb{I}\{j\in\mathcal{S}^{r}\}\sigma^{\prime}(\langle w_{j,l, r}^{(t)}, \Xb^{r}\rangle)z^{r}_{j}+\widetilde{O}(\sigma_0^{q-1})\alpha+O\left(\sigma_{g}\right)\right)\right]
    \end{align*}
By Induction Hypothesis~\ref{hypo},when $(\Xb,y)\in\mathcal{D}_s$ and $y=j$, $\sigma^{\prime}(\langle w_{j,l, r}^{(t)}, \Xb^{r}\rangle)z^{r}_{j}\geq \Omega(1)\sigma^{\prime}(\langle w^{(t)}_{j,l, r}, \Mb^{r}_{j}\rangle)$. When $j\neq y$, and $j\in\mathcal{S}^{r}$, we have $\sigma^{\prime}(\langle w_{j,l, r}^{(t)}, \Xb^{r}\rangle)z^{r}_{j}\leq O(1)\sigma^{\prime}(\langle w^{(t)}_{j,l, r}, \Mb^{r}_{j}\rangle)$. Combining with the fact $\ell_{j}(f^{(t)},\Xb)\leq O(\frac{1}{K})$, we obtain:
\begin{align*}
    \left\langle w_{j, l, r}^{(t+1)}, \Mb_{j}^{r}\right\rangle\geq  \left\langle w_{j, l, r}^{(t)}, \Mb_{j}^{r}\right\rangle+\frac{\eta}{K}(\Omega(1)-o(1))\sigma^{\prime}(\langle w^{(t)}_{j,l, r}, \Mb^{r}_{j}\rangle)-\frac{\eta}{K}\widetilde{O}(\sigma_0^{q-1}+\sigma_g)
\end{align*}
Then, we derive that 
\begin{align*}
    \left\langle w_{j, l, r}^{(t+1)}, \Mb_{j}^{r}\right\rangle\geq  \left\langle w_{j, l, r}^{(t)}, \Mb_{j}^{r}\right\rangle+\frac{\Omega(\eta)}{K}\sigma^{\prime}(\langle w^{(t)}_{j,l, r}, \Mb^{r}_{j}\rangle)
\end{align*}
On the other hand,
\begin{align*}
    \left\langle w_{j, l, r}^{(t+1)}, \Mb_{j}^{r}\right\rangle&\leq  \left\langle w_{j, l, r}^{(t)}, \Mb_{j}^{r}\right\rangle\\
&+\frac{\eta}{n}\sum_{(\Xb,y)\in \mathcal{D}}\left[\mathbb{I}\{y=j\}\left(1-\ell_{j}\left(f^{(t)}, \Xb\right)\right)\left(\sigma^{\prime}(\langle w_{j,l, r}^{(t)}, \Xb^{r}\rangle)z^{r}_{j}+O\left(\sigma_{g}\right)\right)\right.\\
&-\mathbb{I}\{y\neq j\}\left.\ell_{j}\left(f^{(t)}, \Xb\right)\left(\mathbb{I}\{j\in\mathcal{S}^{r}\}\sigma^{\prime}(\langle w_{j,l, r}^{(t)}, \Xb^{r}\rangle)z^{r}_{j}-O\left(\sigma_{g}\right)\right)\right]
\end{align*}
Following the similar analysis, we have
\begin{align*}
    \left\langle w_{j, l, r}^{(t+1)}, \Mb_{j}^{r}\right\rangle\leq  \left\langle w_{j, l, r}^{(t)}, \Mb_{j}^{r}\right\rangle+\frac{O(\eta)}{K}\sigma^{\prime}(\langle w^{(t)}_{j,l, r}, \Mb^{r}_{j}\rangle)
\end{align*}
Hence we complete the proof.
\end{proof}
\paragraph{Training phases.} With the above results,  we decompose the training process into two phases for each class $j\in[K]$: 
\begin{itemize}
    \item Phase 1:  $t\leq T_{j}$, where $T_{j}$ is the iteration number that $\Gamma_{j}^{(t)}$ reaches $\Theta\left(\frac{\beta}{\log k}\right)=\widetilde{\Theta}(1)$ (recall that $\beta$ is the activation function threshold)
    \item Phaes 2, stage 1: $T_{j}\leq t\leq T_{0}$:
     where  $T_0$ denote the iteration number that all of the $\Gamma_{j}^{(t)}$ reaches $\Theta(1/m)$;
    \item Phase 2, stage 2: $t\geq T_0$, i.e. from $T_0$ to the end $T$.
    \end{itemize}
   From Fact~\ref{fact-err}, we observe that the contribution of $j$-th output of $f^{(t)}$ is negligible unless reaches $\Theta(1/m)$, Hence, 
after $T_0$, the output of $f^{(t)}$ is significant
 which represents the network has learned certain partterns, and the training process enters the final convergence stage. By Claim~\ref{claim-grow}, we have $T_0=\Theta(K/\eta \sigma_0^{q-2})$. Note that $T_0\geq T_{j}$, for every $j\in[K]$.

    \subsection{Error Analysis}
   \subsubsection{Error for Insufficient Data}
\begin{claim}[Noise Correlation]~\label{cla-noise}
    \begin{itemize}
        \item [(a)] For every $(\Xb, y) \in \mathcal{D}_{i}$, every $r \in[2]$:
        $$
        \left\langle w_{y, l, r}^{(t+1)}, {\xi^{r}}^{\prime}\right\rangle \geq\left\langle w_{y, l, r}^{(t)}, {\xi^{r}}^{\prime}\right\rangle-\frac{\eta}{\sqrt{d_r}}+\widetilde{\Omega}\left(\frac{\eta}{n}\right) \sigma^{\prime}\left(\left\langle w_{y, l, r}^{(t)}, \Xb^{r}\right\rangle\right)\left(1-\ell_{y}\left(f^{(t)}, \Xb\right)\right) \geq \cdots \geq-\frac{\eta T}{\sqrt{d_r}}
        $$
        \item [(b)] For every $(\Xb, y) \in \mathcal{D}_{i}$, every $r\in[2]$,
        $$
        \begin{aligned}
        \left\langle w_{y, l, r}^{(t+1)}, {\xi^{r}}^{\prime}\right\rangle \geq & \left\langle w_{y, l, r}^{(t)}, {\xi^{r}}^{\prime}\right\rangle-\frac{\eta}{\sqrt{d_r}}\\
        &+\widetilde{\Omega}\left(\frac{\eta}{n}\right) \sigma^{\prime}\left(\Theta(\gamma_r) \cdot\left\langle w_{y, l, r}^{(t)}, M_{y }^{r}\right\rangle-\widetilde{O}\left(\frac{\eta T} { \sqrt{d_r}}+\sigma_{0} \alpha K\right)\right)\left(1-\ell_{y}\left(f^{(t)}, \Xb\right)\right)
        \end{aligned}
        $$
    \end{itemize}

\end{claim}
 
\begin{proof}
    For $(\Xb_0,y_0)\in \mathcal{D}_i$
    $$
    \begin{aligned}
    \left\langle w_{j, l, r}^{(t+1)}, {\xi_0^{r}}^{\prime}\right\rangle&=\left\langle w_{j, l, r}^{(t)}, {\xi_0^{r}}^{\prime}\right\rangle\\&+\frac{\eta}{n} \sum_{(\Xb, y)\in \mathcal{D}}\left[\mathbb{I}\{y=j\}\sigma^{\prime}\left(\left\langle w_{j, l, r}^{(t)}, X^{r}\right\rangle\right)\left\langle X^{r}, {\xi_0^{r}}^{\prime}\right\rangle\left(1-\ell_{j}\left(f^{(t)}, \Xb\right)\right)\right.\\
    &-\mathbb{I}\{y \neq j\} \left.\sigma^{\prime}\left(\left\langle w_{j, l,r}^{(t)}, X^{r}\right\rangle\right)\left\langle X^{r}, {\xi_0^{r}}^{\prime}\right\rangle \ell_{j}\left(f^{(t)}, \Xb\right)\right]
    \end{aligned}
    $$
    If $j=y_0$, $|\langle X^{r}, {\xi_0^{r}}^{\prime}\rangle|\leq \widetilde{O}(\sigma_g)=\widetilde{o}(\frac{1}{\sqrt{d_r}})$ except for $X^{r}_0$, then we have:
$$
\begin{aligned}
\left\langle w_{j, l, r}^{(t+1)}, {\xi_0^{r}}^{\prime}\right\rangle=\left\langle w_{j, l, r}^{(t)}, {\xi_0^{r}}^{\prime}\right\rangle\pm\frac{\eta}{\sqrt{d_r}}+\widetilde{\Theta}(\frac{\eta}{n})  \sigma^{\prime}\left(\left\langle w_{j, l, r}^{(t)}, X_0^{r}\right\rangle\right)\left(1-\ell_{j}\left(f^{(t)}, \Xb_0\right)\right)
\end{aligned}
$$
By the non-negativity of $\sigma^{\prime}$, we prove the first claim. Furthermore, by induction hypothesis, 
$$\left\langle w_{y, l, r}^{(t)}, \Xb^{r}\right\rangle=\left\langle w_{y, l, r}^{(t)}, \Mb_{y}^{r}\right\rangle z^{r}_{y}+\left\langle w_{y, l, r}^{(t)}, {\xi^{r}}^{\prime}\right\rangle \pm \widetilde{O}\left(\sigma_{0} \alpha K\right)\geq \Theta(\gamma_r) \left\langle w_{y, l, r}^{(t)}, \Mb_{y}^{r}\right\rangle-\frac{\eta T}{\sqrt{d_r}}-\widetilde{O}(\sigma_0\alpha K)$$
we complete the proof.
\end{proof}

\begin{claim}[Error for Insufficient Data]~\label{cla-inerr}
    Suppose Induction Hypothesis~\ref{hypo}  holds  for all iterations $t<T$ and $\alpha \leq \widetilde{O}\left(\sigma_{0} K\right) .$ We have that
    \begin{itemize}
        \item [(a)] for every $(\Xb, y) \in \mathcal{D}_{i}$, for every $l \in[m]$, every $r \in[2]$:
        $$
        \sum_{t=T_{0}}^{T}\left(1-\ell_{y}\left(f^{(t)}, \Xb\right)\right) \sigma^{\prime}\left(\left\langle w^{(t)}_{y, l, r}, \Xb^{r}\right\rangle\right) \leq \widetilde{O}\left(\frac{n}{\eta}\right)
        $$
        \item [(b)] for every $(\Xb, y) \in \mathcal{D}_{i}$, 
        $$
        \sum_{t=T_{0}}^{T}\left(1-\ell_{y}\left(f^{(t)}, \Xb\right)\right) \leq \widetilde{O}\left(\frac{n}{\eta \gamma^{q-1}}\right)
        $$
    \end{itemize}
\end{claim}

\begin{proof}
    Once $\sum_{t=T_{0}}^{T^{\prime}}\left(1-\ell_{y}\left(f^{(t)}, \Xb\right)\right) \sigma^{\prime}\left(\left\langle w_{y, l, r}, \Xb^{r}\right\rangle\right) $ reaches $\widetilde{\Theta}\left(\frac{n}{\eta}\right)$ for some $T^{\prime}\leq T$, 
    by Claim~\ref{cla-noise}, for $t\geq T^{\prime}$
    \begin{align*}
        \left\langle w_{y, l, r}^{(t)}, {\xi^{r}}^{\prime}\right\rangle \geq \widetilde{O}(1)-\frac{1}{\poly(K)}=\polylog(K)
    \end{align*}
    Hence, $f^{(t)}_{y}(\Xb)\geq \left\langle w_{y, l, r}, \Xb^{r}\right\rangle \geq \polylog(K)$. And for $j\neq y$, $f^{(t)}_{j}(\Xb)\leq m 
\Gamma^{(t)}_{j}(\rho_1+\rho_2) \leq O(1)$. Therefore, $1-\ell_{y}(F^{(t)},\Xb)\leq \exp(-\polylog(K))=O(\frac{1}{\poly(K)})$, and the summation cannot further
    exceed $\widetilde{O}(\frac{n}{\eta})=\widetilde{O}(\poly(K))$.
    
    
    For (b),  suppose $\sum_{t=T_{0}}^{T}\left(1-\ell_{y}\left(f^{(t)}, \Xb\right)\right) \geq \widetilde{\Omega}\left(\frac{n}{\eta \gamma^{q-1}}\right)$. 
    Since $\Gamma^{(t)}_{j}\geq \widetilde{\Omega}(1)$, by averaging we have:
    \begin{align*}
        \sum_{l\in[m]}\sum_{r\in[2]} \mathbb{I}\{\left\langle w_{y, l, r}^{(t)}, {\Mb_{y}^{r}}\right\rangle\geq\widetilde{\Omega}(1)\}\sum_{t=T_{0}}^{T}\left(1-\ell_{y}\left(f^{(t)}, \Xb\right)\right) \geq \widetilde{\Omega}\left(\frac{n}{\eta \gamma^{q-1}}\right)
    \end{align*}

    When $\left\langle w_{y, l, r}^{(t)}, {\xi^{r}}^{\prime}\right\rangle\geq \polylog(K)$ and $\left\langle w_{y, l, r}^{(t)}, {\Mb_{y}^{r}}\right\rangle\geq\widetilde{\Omega}(1)$ simultaneously holds, from the above analysis, we have $1-\ell_{y}(F^{(t)},\Xb)\leq \exp(-\polylog(K))$, hence we only consider the case $\left\langle w_{y, l, r}^{(t)}, {\xi^{r}}^{\prime}\right\rangle\leq \polylog(K)$. We decompose $[T_0,T]$
    into $2m+1$ interval, which is denoted by $\tau_1,\cdots,\tau_{2m+1}$, s.t.  
    \begin{align*}
        \sum_{t\in\tau_i}\sum_{l\in[m]}\sum_{r\in[2]} \mathbb{I}\{\left\langle w_{y, l, r}^{(t)}, {\Mb_{y}^{r}}\right\rangle\geq\widetilde{\Omega}(1),\left\langle w_{y, l, r}^{(t)}, {\xi^{r}}^{\prime}\right\rangle\leq \polylog(K)\}\left(1-\ell_{y}\left(f^{(t)}, \Xb\right)\right) \geq \widetilde{\Omega}\left(\frac{n}{\eta \gamma^{q-1}}\right)
    \end{align*}
    for every $i=1,\cdots,2m+1.$ By averaging, there exists $(l_1,r_1)\in[m]\times [2],s.t.$ 
    \begin{align*}
        \sum_{t\in\tau_1}\mathbb{I}\{\left\langle w_{y, l_1, r_1}^{(t)}, {\Mb_{y}^{r_1}}\right\rangle\geq\widetilde{\Omega}(1),\left\langle w_{y, l_1, r_1}^{(t)}, {\xi^{r_1}}\right\rangle\leq \polylog(K) \}\left(1-\ell_{y}\left(f^{(t)}, \Xb\right)\right) \geq \widetilde{\Omega}\left(\frac{n}{\eta \gamma^{q-1}}\right)
    \end{align*}
    By Claim~\ref{cla-noise} $(b)$, we obtain, for $t\notin \tau_1$, 
    $$
    \left\langle w_{y, l_1, r_1}^{(t)}, {\xi^{r_1}}^{\prime}\right\rangle\geq 
    \widetilde{\Omega}\left(\frac{n}{\eta \gamma^{q-1}}\right)\cdot \widetilde{\Omega}\left(\frac{\eta}{n}\right)\cdot \gamma^{q-1}=\widetilde{\Omega}(1)
    $$
    Similarly, there exists $(l_2,r_2)\in[m]\times [2],s.t.$ 
    \begin{align*}
        \sum_{t\in\tau_2}\mathbb{I}\{\left\langle w_{y, l_2, r_2}^{(t)}, {\Mb_{y}^{r_2}}\right\rangle\geq\widetilde{\Omega}(1),\left\langle w_{y, l_2, r_2}^{(t)}, {\xi^{r_2}}^{\prime}\right\rangle\leq \polylog(K) \}\left(1-\ell_{y}\left(f^{(t)}, \Xb\right)\right) \geq \widetilde{\Omega}\left(\frac{n}{\eta \gamma^{q-1}}\right)
    \end{align*}
    Clearly, $(l_2,r_2)\neq (l_1,r_1)$. Keep the similar procedure, we obtain for $t\in\tau_{2m+1}$, $\left\langle w_{y, l,r}^{(t)}, {\xi^{r}}^{\prime}\right\rangle\geq \polylog(K)$ for all $(l,r)\in[m]\times[2]$, which contradicts the fact that
    \begin{align*}
        \sum_{t\in\tau_{2m+1}}\sum_{l\in[m]}\sum_{r\in[2]} \mathbb{I}\{\left\langle w_{y, l, r}^{(t)}, {\Mb_{y}^{r}}\right\rangle\geq\widetilde{\Omega}(1),\left\langle w_{y, l, r}^{(t)}, {\xi^{r}}^{\prime}\right\rangle\leq \polylog(K)\}\left(1-\ell_{y}\left(f^{(t)}, \Xb\right)\right) \geq \widetilde{\Omega}\left(\frac{n}{\eta \gamma^{q-1}}\right)
    \end{align*}
    Therefore, we prove $\sum_{t=T_{0}}^{T}\left(1-\ell_{y}\left(f^{(t)}, \Xb\right)\right) \leq \widetilde{O}\left(\frac{n}{\eta \gamma^{q-1}}\right)$. 
    \end{proof}

    \subsubsection{Error for Sufficient Data}
    \begin{claim}[Individual Error]~\label{cla-individual}
        For every $t \geq 0$, every $(\Xb, y) \in \mathcal{D}_{s}$, we have
$$
1-\ell_{y}\left(f^{(t)}, \Xb\right) \leq \tilde{O}\left(\frac{K^3}{s^2}\right) \cdot \frac{1}{n_s} \sum_{(\Xb,y)\in \mathcal{D}_s}\left[1-\ell_{y}\left(f^{(t)}, \Xb\right)\right]
$$
    \end{claim}
    
    \begin{proof}
        It is easy to verify that
        $$
        1-\frac{1}{1+x}\leq \min\{1,x\}\leq 2(1-\frac{1}{1+x})
        $$
        On the one hand, for $(\Xb,y)\in \mathcal{D}_{s}$,  we have
        $$
        \begin{aligned}
            1-\ell_{y}\left(f^{(t)}, \Xb\right) &\leq  \min\{1,\sum_{j\neq y} \exp(\max\{c_1,c_2\}\Phi^{(t)}_{j}-\Phi_{y}^{(t)})\}\leq \sum_{j\neq y} \min\{1/K,\exp(\max\{c_1,c_2\}\Phi^{(t)}_{j}-\Phi_{y}^{(t)})\}\\
           & \leq \sum_{i\in [K]}\sum_{j\neq i} \min\{1/K,\exp(\max\{c_1,c_2\}\Phi^{(t)}_{j}-\Phi_{i}^{(t)})\}
        \end{aligned}
        $$
        Moreover, 
        \begin{align*}
           & \frac{1}{n_s} \sum_{(\Xb,y)\in \mathcal{D}_s} [1-\ell_{y}\left(f^{(t)}, \Xb\right)]\geq \frac{1}{2n_s} \sum_{(\Xb,y)\in \mathcal{D}_s} \min\{1,\sum_{j\neq y}\exp(F_j^{(t)}(X)-F_y^{(t)}(X))\}\\      
           &\geq  \frac{1}{2n_s} \sum_{(\Xb,y)\in \mathcal{D}_s} \min\{1,\sum_{j\in \mathcal{S}^{1}(X)\cap \mathcal{S}^{2}(X) }\exp(\max\{c_1,c_2\}\Phi^{(t)}_{j}-\Phi_{y}^{(t)})\}\\      
           & \geq \frac{1}{2n_s} \sum_{(\Xb,y)\in \mathcal{D}_s}\sum_{j\in \mathcal{S}^{1}(X)\cap \mathcal{S}^{2}(X)} \min\{1/K,\exp(\max\{c_1,c_2\}\Phi^{(t)}_{j}-\Phi_{y}^{(t)})\}\\
           & =\sum_{i\in [K]}\sum_{j\in[K]}  \frac{1}{2n_s} \sum_{(\Xb,y)\in \mathcal{D}_s} \mathbb{I}\{i=y\}\mathbb{I}\{j\in \mathcal{S}^{1}(X)\cap \mathcal{S}^{2}(X)\} \min\{1/K,\exp(\max\{c_1,c_2\}\Phi^{(t)}_{j}-\Phi_{i}^{(t)})\}\\
           &\geq \widetilde{\Omega}(\frac{s^2}{K^3}) \sum_{i\in [K]}\sum_{j\in[K], j\neq i}  \min\{1/K,\exp(\max\{c_1,c_2\}\Phi^{(t)}_{j}-\Phi_{i}^{(t)})\}
        \end{align*}
        Therefore,
        $$
        \begin{aligned}
            1-\ell_{y}\left(f^{(t)}, \Xb\right) \leq  
            \tilde{O}(\frac{K^3}{s^2})\frac{1}{n_s} \sum_{(\Xb,y)\in \mathcal{D}_s}\left[1-\ell_{y}\left(f^{(t)}, \Xb\right)\right]
        \end{aligned}
        $$

        \end{proof}

        \begin{claim}[Phase 2, Stage 2]~\label{cla-2.2}
      
            For every $(\Xb, y) \in \mathcal{D}_{s}$, every $t\geq T_0$
            $$\begin{aligned} \sum_{j \in[K]} \Gamma_{j}^{(t+1)} &\geq \sum_{j \in[K]} \Gamma_{j}^{(t)}+\Omega(\eta) \times \frac{1}{n_s}\underset{(\Xb, y) \in \mathcal{D}_{s}}{\sum}\left[1-\ell_{y}\left(f^{(t)}, \Xb\right)\right] \\ &- O\left( \frac{\eta s n_{i}}{Kn}\right) \frac{1}{n_i}\underset{(\Xb, y) \in \mathcal{D}_{i}}{\sum}\left[1-\ell_{y}\left(f^{(t)}, \Xb\right)\right] \end{aligned}$$
            Denote:
            $$
            Err^{\text{Tol, Stage 3 }}_{s}:=\sum_{t\geq T_0}\frac{1}{n_s}\sum_{(\Xb,y)\in \mathcal{D}_s} \left(1-\ell_{y}\left(f^{(t)}, \Xb\right)\right) 
            $$
            Consequently, we have
            $$
            Err^{\text{Tol, Stage 3 }}_{s}\leq \widetilde{O}\left(\frac{K}{\eta}\right)+\widetilde{O}\left(\frac{n_{i} s}{\eta K \gamma^{q-1}}\right) 
            $$
        \end{claim}

        \begin{proof}
                
                Let $(l,r)=\arg\max_{l\in[m],r\in[2]}[\langle w_{j, l, r}^{(t)}, {\Mb_{j}^{r}}\rangle]^{+}$. 
                By gradient updates, we have
                \begin{align*}
                    \left\langle w_{j, l, r}^{(t+1)}, \Mb_{j}^{r}\right\rangle&\geq  \left\langle w_{j, l, r}^{(t)}, \Mb_{j}^{r}\right\rangle\\
                &+\frac{\eta}{n}\sum_{(\Xb,y)\in \mathcal{D}}\left[\mathbb{I}\{y=j\}\left(1-\ell_{j}\left(f^{(t)}, \Xb\right)\right)\left(\sigma^{\prime}(\langle w_{j,l, r}^{(t)}, \Xb^{r}\rangle)z^{r}_{j}-O\left(\sigma_{g}\right)\right)\right.\\
                &-\mathbb{I}\{y\neq j\}\left.\ell_{j}\left(f^{(t)}, \Xb\right)\left(\sigma^{\prime}(\langle w_{j,l, r}^{(t)}, \Xb^{r}\rangle)\mathbb{I}\{j\in\mathcal{S}^{r}(X)\}z^{r}_{j}+\widetilde{O}(\sigma^{q-1}_0)\alpha+ O\left(\sigma_{g}\right)\right)\right]
                \end{align*}
                In the Stage $3$,  $\langle w_{j, l, r}^{(t)}, {\Mb_{j}^{r}}\rangle\geq \widetilde{\Theta}(1)\gg \beta$
                \begin{itemize}
                    \item For sufficient multi-modal data, when $j=y$ or $j\in\mathcal{S}^{r}(X)$,  $\langle w_{j,l}^{(t)}, X^{r}\rangle=\langle w_{j,l}^{(t)}, \Mb_{j}^{r}\rangle z^{r}_{j}\pm\widetilde{o}(\sigma_0)$, hence $\langle w_{j,l}^{(t)}, \Xb^{r}\rangle$ is already in the linear regime of activation function:  
                    \begin{itemize}
                        \item For $j=y$, $z^{r}_{j}\in [1, C]\Rightarrow \sigma^{\prime}(\langle w_{j,l, r}^{(t)}, \Xb^{r}\rangle)z^{r}_{j}\geq(1-o(1)) z^{r}_{j}\geq 1-o(1)$ 
                        \item For $j\in\mathcal{S}^{r}(X)$, $z^{r}_{j}\in [\Omega(1), c_r]\Rightarrow \sigma^{\prime}(\langle w_{j,l, r}^{(t)}, \Xb^{r}\rangle)z^{r}_{j}\leq c_r$
                    \end{itemize}
                    \item For insufficient multi-modal data:
                    \begin{itemize}
                        \item For $j=y$, $\sigma^{\prime}(\langle w_{j,l,r}^{(t)}, X^{r}\rangle)z^{r}_{j}$ has naive lower bound $0$.
                        \item For $j\in\mathcal{S}^{r}(X)$,   we have $\sigma^{\prime}(\langle w_{j,l, r}^{(t)}, \Xb^{r}\rangle)\leq \rho_r$, and $\ell_{j}\left(f^{(t)}, \Xb\right)=O\left(\frac{1}{K}\right)\left(1-\ell_{y}\left(f^{(t)}, \Xb\right)\right)$.
                    \end{itemize}
                \end{itemize}
                Therefore
                \begin{align}
                   & \left\langle w_{j, l, r}^{(t+1)}, \Mb_{j}^{r}\right\rangle\geq  \left\langle w_{j, l, r}^{(t)}, \Mb_{j}^{r}\right\rangle\nonumber\\
                &+\frac{\eta}{n_s}\sum_{(\Xb,y)\in \mathcal{D}_s}\left[\mathbb{I}\{y=j\}(1-o(1))(1-\ell_{y}\left(f^{(t)}, \Xb\right))-\mathbb{I}\{y\neq j\} c_r \ell_{j}\left(f^{(t)}, \Xb\right)\right]\nonumber\\
                &- \frac{\eta n_i}{Kn}\cdot \frac{1}{n_i}\sum_{(\Xb,y)\in\mathcal{D}_i} \left[\left(\mathbb{I}\{y=j\}O\left(K\sigma_{g}\right)+\mathbb{I}\{y\neq j\}(O\left(\sigma_{g}\right)+\mathbb{I}\{j\in\mathcal{S}^{r}(X)\})\right)(1-\ell_{y}\left(f^{(t)}, \Xb\right))\right]
                 \label{gs3}
                \end{align}
                Summing over $j\in [K]$, we have:
                $$
        \begin{aligned}
        \sum_{j \in[K]} \Gamma_{j}^{(t+1)} \geq \sum_{j \in[K]} \Gamma_{j}^{(t)} &+\Omega(\eta) \times \frac{1}{n_s}\underset{(\Xb, y) \in \mathcal{D}_{s}}{\sum}\left[1-\ell_{y}\left(f^{(t)}, \Xb\right)\right] \\
        &-\eta O\left(\frac{s}{K} \frac{n_{i}}{n}\right) \times \frac{1}{n_i}\underset{(\Xb, y) \in \mathcal{D}_{i}}{\sum}\left[1-\ell_{y}\left(f^{(t)}, \Xb\right)\right]
        \end{aligned}
        $$
                \end{proof}  

\begin{claim}[Phase 2, Stage 1]~\label{cla-2.1}
    Denote:
$$
Err^{\text{Tol, Stage 2 }}_{s,j}:=\sum^{T_0}_{t=T_j}\frac{1}{n_s}\sum_{(\Xb,y)\in \mathcal{D}_s} \mathbb{I}\{y=j\} \left(1-\ell_{y}\left(f^{(t)}, \Xb\right)\right) 
$$
$$
\widetilde{Err}^{\text{Stage 2}}_{s,j}:=\frac{1}{n_s}\sum_{(\Xb,y)\in \mathcal{D}_s} \mathbb{I}\{y\neq j\} \ell_{j}\left(f^{(t)}, \Xb\right)
$$
    For every $(\Xb, y) \in \mathcal{D}_{s}$, every $T_0\geq t \geq T_j$, we have
    \begin{enumerate}
        \item [1)] for $\Lambda\in[\frac{1}{K}, \frac{1}{s}]$, $\Lambda \leq \widetilde{O}(K^{1-2c})$
    $$
    Err^{\text{Tol, Stage 2 }}_{s,j}\leq  \widetilde{O}(\frac{1}{\eta})+O(\frac{s \Lambda}{K}T_0)
    $$
    \item [2)] for every $t \in\left[T_{j}, T_{0}\right]$,
    $$
    \begin{gathered}
       \widetilde{Err}^{\text{Stage 2}}_{s,j}\leq O\left(\frac{1}{K}\right)
    \end{gathered}
    $$
    \end{enumerate}
\end{claim}
In order to prove Claim~\ref{cla-2.1}, let us first prove the following lemma:
\begin{lemma}
    Consider $\Lambda \in [\frac{1}{K}, \frac{1}{s}]$, letting $T^{*}:=\widetilde{\Theta}(k^{\frac{1}{c}}\Lambda^{\frac{1}{c}}/\eta)$, where $c:=\{c_1,c_2\}$,  then we have $t\leq T^{*}$, $\exp(c\Phi^{t}_{j})\leq k\Lambda$ for any $j\in [K]$.
\end{lemma}
\begin{proof}
    Denote 
    $$
\overline{\Phi}^{(t)}=\max _{j \in[K]} \sum_{l\in [m]}\sum_{r\in [2]}\left[\left\langle w_{j, l, r}^{(t)}, \Mb^{r}_{j}\right\rangle\right]^{+}
$$
Let $j^{*}:=\arg\max _{j \in[K]} \sum_{l\in [m]}\sum_{r\in [2]}\left[\left\langle w_{j, l, r}^{(t)}, M^{r}_{j}\right\rangle\right]^{+}$. By gradient updates, we have:
\begin{align}
& \left\langle w_{j^{*}, l, r}^{(t+1)}, M_{j^{*}}^{r}\right\rangle\leq  \left\langle w_{j^{*}, l, r}^{(t)}, M_{j^{*}}^{r}\right\rangle\nonumber\\
&+\frac{\eta}{n}\sum_{(\Xb,y)\in \mathcal{D}}\left[\mathbb{I}\{y=j^{*}\}(\sigma^{\prime}(\langle w_{j^{*},l, r}^{(t)}, X^{r}\rangle)z^{r}_{j^{*}}+O(\sigma_g))(1-\ell_{y}\left(f^{(t)}, \Xb\right))+\mathbb{I}\{y\neq j^{*}\} O(\sigma_g)\ell_{j^{*}}\left(f^{(t)}, \Xb\right)\right]\\
&\leq \left\langle w_{j^{*}, l, r}^{(t)}, M_{j^{*}}^{r}\right\rangle\nonumber+O(\eta)(\frac{1}{n}\sum_{(\Xb,y)\in \mathcal{D}} \mathbb{I}\{y=j^{*}\}(1-\ell_{y}\left(f^{(t)}, \Xb\right))+O(\sigma_g))
\end{align}
We only focus on the $\mathcal{D}_s$ since the contribution of insufficient data is negligible.
\begin{itemize}
\item For $j=y$, $f_{y}^{(t)}(\Xb) \geq \Phi_{y}^{(t)}-\frac{1}{\text { polylog }(K)}$, w.p. $\frac{1}{K}$
\item For $j\in \mathcal{S}^{1}(\Xb)\cup\mathcal{S}^{2}(\Xb)$, $f_{j}^{(t)}(\Xb) \leq  c \Phi_{j}^{(t)}+\frac{1}{\text { polylog }(K)}$, w.p. $(1-\frac{s}{K})^2$
\item Else, $f_{j}^{(t)}(\Xb) \leq \frac{1}{\text { polylog }(k)}$, w.p. $1-(1-\frac{s}{K})^2$
\end{itemize}
Then we obtain:
$$
\frac{1}{n}\sum_{(\Xb,y)\in \mathcal{D}_s} \mathbb{I}\{y=j^{*}\}(1-\ell_{y}\left(f^{(t)}, \Xb\right))\leq \frac{1}{n}\sum_{(\Xb,y)\in \mathcal{D}_s} \mathbb{I}\{y=j^{*}\}\frac{\sum_{j \neq y} e^{f_{j}^{(t)}(\Xb)}}{e^{f_{y}^{(t)}(\Xb)}}\leq \frac{1}{K}O(\frac{K+s\exp(c\overline{\Phi}^{(t)})}{\exp(\overline{\Phi}^{(t)})})
$$
Summing over $(r,l)$, we have:
$$
\overline{\Phi}^{(t+1)}\leq \overline{\Phi}^{(t)}+ \frac{\eta}{K}\widetilde{O}(\frac{n_i}{n}+ \frac{K+s\exp(c\overline{\Phi}^{(t)})}{\exp(\overline{\Phi}^{(t)})})
$$
Once $\exp(\overline{\Phi}^{(t)})$ reaches $\Omega(k^{\frac{1}{c}}\Lambda^{\frac{1}{c}})$, then $\overline{\Phi}^{(t+1)}\leq \overline{\Phi}^{(t)}+\eta \widetilde{O}(k^{-\frac{1}{c}}\Lambda^{-\frac{1}{c}})  $, which implies $\exp(c\overline{\Phi}^{(t+1)})$ cannot further exceed $k\Lambda$
\end{proof}
\begin{proof}[Proof of Claim~\ref{cla-2.1}]
    Following the similar gradient analysis in (\ref{gs3}), we have

\begin{align}
        \Gamma^{(t+1)}_j&\geq \Gamma^{(t)}_j\\
        &+\frac{\eta}{n_s}\sum_{(\Xb,y)\in \mathcal{D}_s}\left[\mathbb{I}\{y=j\}(1-o(1))(1-\ell_{y}\left(f^{(t)}, \Xb\right))\right.\nonumber\\
        &\left.-\mathbb{I}\{y\neq j\} (c\mathbb{I}\{j\in \mathcal{S}^{1}(\Xb)\cup \mathcal{S}^{2}(\Xb)\}+\widetilde{O}(\sigma_0^{q-1})\alpha+ O(\sigma_g)) \ell_{j}\left(f^{(t)}, \Xb\right)\right]\nonumber\\
&- \frac{\eta n_i}{Kn}\cdot \frac{1}{n_i}\sum_{(\Xb,y)\in\mathcal{D}_i} \left[\left(\mathbb{I}\{y=j\}O\left(K\sigma_{g}\right)\right.\right.\nonumber\\
&\left. \left.+\mathbb{I}\{y\neq j\}(O\left(\sigma_{g}\right)+\widetilde{O}(\sigma_0^{q-1})\alpha+\mathbb{I}\{j\in\mathcal{S}^{1}(\Xb)\cup \mathcal{S}^{2}(\Xb)\})\right)(1-\ell_{y}\left(f^{(t)}, \Xb\right))\right]\label{gp2}
    \end{align}
For $(\Xb,y)\in\mathcal{D}_s$, and $j\in\mathcal{S}^{1}(\Xb)\cup \mathcal{S}^{2}(\Xb)$, we easily derive that $f_{j}^{(t)}(\Xb) \leq c \Phi_{j}^{(t)}+\frac{1}{\text { polylog }(K)}$. Hence, 
\begin{align*}
    &\frac{1}{n_s}\sum_{(\Xb,y)\in \mathcal{D}_s}\left[\mathbb{I}\{y\neq j\} \mathbb{I}\{j\in \mathcal{S}^{1}(\Xb)\cup \mathcal{S}^{2}(\Xb)\} \ell_{j}\left(f^{(t)}, \Xb\right)\right]\nonumber\\
    &= \frac{1}{n_s}\sum_{(\Xb,y)\in \mathcal{D}_s}\left[\mathbb{I}\{y\neq j\} \mathbb{I}\{j\in \mathcal{S}^{1}(\Xb)\cup \mathcal{S}^{2}(\Xb)\}\frac{1}{1+\sum_{i\neq j}\exp(f_{i}(\Xb)-f_{j}(\Xb)) } \right]\\
    &\leq \frac{1}{n_s}\sum_{(\Xb,y)\in \mathcal{D}_s}\left[\mathbb{I}\{y\neq j\} \mathbb{I}\{j\in \mathcal{S}^{1}(\Xb)\cup \mathcal{S}^{2}(\Xb)\}\frac{1}{1+\sum_{i\neq j}\exp(f^{(t)}_{i}(\Xb)-c \Phi_{j}^{(t)}) } \right]
\end{align*}
If we let $\Lambda =\widetilde{\Theta}(K^{2c-1})$, then $T^{*}\geq T_0$. By the above lemma, we have 
$$
\frac{1}{n_s}\sum_{(\Xb,y)\in \mathcal{D}_s}\left[\mathbb{I}\{y\neq j\} \mathbb{I}\{j\in \mathcal{S}^{1}(\Xb)\cup \mathcal{S}^{2}(\Xb)\} \ell_{j}\left(f^{(t)}, \Xb\right)\right]\leq O(\Lambda)
$$
Taking back into $(\ref{gp2})$:
$$
\Gamma^{(t+1)}_j\geq \Gamma^{(t)}_j+\Omega(\eta)(\frac{1}{n_s}\sum_{(\Xb,y)\in \mathcal{D}_s}\left[\mathbb{I}\{y=j\}(1-\ell_{y}\left(f^{(t)}, \Xb\right))\right]-O(\frac{n_i}{n}\cdot \frac{s}{K^2})-O(\frac{s \Lambda}{K}))
$$
Combining with the fact that $\Gamma^{(t)}_j\leq \widetilde{O}(1)$, we finally derive that 
$$
\sum^{T_0}_{t=T_j}\frac{1}{n_s}\sum_{(\Xb,y)\in \mathcal{D}_s}\left[\mathbb{I}\{y=j\}(1-\ell_{y}\left(f^{(t)}, \Xb\right))\right]\leq \widetilde{O}(\frac{1}{\eta})+O(\frac{s \Lambda}{K}T_0)
$$

\end{proof}

    \subsection{Modality Competition}
Define a data-dependent parameter:
\begin{align*}
    d_{j, r}(\mathcal{D})=\frac{1}{n\beta^{q-1}}\sum_ {(\Xb,y)\in\mathcal{D}_{s}}\mathbb{I}\{y=j\}  \left(z^{r}_{j}\right)^{q}
\end{align*}
\begin{lemma}\label{lemma-comp}
Denote:
$$
\mathcal{W} \stackrel{\text { def }}{=}\left\{\left(j, r_j\right) \in[K] \times[2] \mid \Gamma_{j, r_j}^{(0)}d_{j, r_j}(\mathcal{D})^{\frac{1}{q-2}} \geq \Gamma_{j, 3-r_j}^{(0)}d_{j, 3-r_j}(\mathcal{D})^{\frac{1}{q-2}}(1+\frac{1}{\operatorname{polylog}(K)})\right\}
$$
$\mathcal{W}$ represents the collection of the class and modality pairs to indicate the winning modality of every class.
 Suppose Induction Hypothesis~\ref{hypo}  holds for all iterations $<t$. Then,
$$
\forall (j, r_j) \in \mathcal{W}: \quad \Gamma_{j, 3-r_j}^{(t)}\leq \widetilde{O}\left(\sigma_{0}\right)
$$
\end{lemma}
In order to prove Claim~\ref{lemma-comp}, we introduce a classic result in tensor power analyis~\cite{anandkumar2015analyzing,allen2020towards}:
\begin{lemma}[Tensor Power Bound]~\label{tensor}
    Let $\left\{x_{t}, y_{t}\right\}_{t=1, \ldots}$ be two positive sequences that satisfy
$$
\begin{aligned}
&x_{t+1} \geq x_{t}+\eta \cdot A_t x_{t}^{q-1}\quad\text{for some } A_t=\Theta(1) \\
&y_{t+1} \leq y_{t}+\eta \cdot B_t y_{t}^{q-1}\quad \text{ where } B_t=A_t M \text{ and } M=\Theta(1)\text{ is a constant}
\end{aligned}
$$
Moreover, if $x_{0} \geq y_{0} M^{\frac{1}{q-2}}\left(1+\frac{1}{\operatorname{polylog}(k)}\right)$. For every $C \in\left[x_{0}, O(1)\right]$, let $T_{x}$ be the first iteration such that $x_{t} \geq C$, then we have 
$$
y_{T_x}\leq \widetilde{O}(x_0)
$$
\end{lemma}
\begin{proof}
    By gradient updates, we have:
    \begin{align}
        \left\langle w_{j, l, r}^{(t+1)}, \Mb_{j}^{r}\right\rangle&= \left\langle w_{j, l, r}^{(t)}, \Mb_{j}^{r}\right\rangle\nonumber\\
    &+\frac{\eta}{n}\sum_{(\Xb,y)\in \mathcal{S}}\left[\mathbb{I}\{y=j\}\left(1-\ell_{j}\left(f^{(t)}, \Xb\right)\right)\left(\sigma^{\prime}(\langle w_{j,l, r}^{(t)}, \Xb^{r}\rangle)z^{r}_{j}\pm O\left(\sigma_{g}\right)\right)\right.\nonumber\\
    &-\mathbb{I}\{y\neq j\}\left.\ell_{j}\left(f^{(t)}, \Xb\right)\left(\mathbb{I}\{j\in\mathcal{S}^{r}\}\sigma^{\prime}(\langle w_{j,l,r}^{(t)}, \Xb^{r}\rangle)z^{r}_{j}\pm \widetilde{O}\left(\sigma_0^{q-1}\alpha+\sigma_{g}\right)\right)\right]\label{gu}
    \end{align}
\begin{itemize}
    \item Phase 1:
    for $t\leq T_j$,  we have $\ell_j(f^{t}, \Xb)\leq O(\frac{1}{K})$. Since $n_i\ll n$, we only consider the sufficient multi-modal data in this phase, and we can simplify the above equation into:
    \begin{align*}
    \left\langle w_{j, l, r}^{(t+1)}, \Mb_{j}^{r}\right\rangle&= \left\langle w_{j, l, r}^{(t)}, \Mb_{j}^{r}\right\rangle+\frac{\eta}{n}\sum_{(\Xb,y)\in \mathcal{D}_{s}}\left[ \mathbb{I}\{y=j\}\left(1-O(\frac{1}{K})\right)\sigma^{\prime}(\langle w_{j,l,r}^{(t)}, \Xb^{r}\rangle)z^{r}_{j}\right.\\
    &\left.+\mathbb{I}\{y\neq j\}\mathbb{I}\{j\in\mathcal{S}^{r}\}O(\frac{1}{K})\sigma^{\prime}(\langle w_{j,l, r}^{(t)}, \Xb^{r}\rangle)z^{r}_{j}\pm\widetilde{O}(\frac{\sigma_0\alpha+\sigma_g}{K})\right]
\end{align*}
 When $j=y$ or $j\in \mathcal{S}^{r}$, we have $\langle w_{j,l, r}^{(t)}, \Xb^{r}\rangle= \langle w_{j,l, r}^{(t)}, \Mb_j^{r}\rangle z^{r}_{j}\pm \widetilde{o}(\sigma_0)$. Since we are in Phase 1,  $\langle w_{j,l, r}^{(t)}, \Mb_j^{r}\rangle z_{j}^{r} \ll \beta $
    , then we obtain $\sigma^{\prime}(\langle w_{j,l, r}^{(t)}, \Xb^{r}\rangle)z^{r}_{j}=[\langle w_{j,l, r}^{(t)}, \Mb_j^{r}\rangle^{+}]^{q-1}(z^{r}_{j})^{q}/\beta^{q-1}\pm\widetilde{O}(\sigma_0)$. Hence
        \begin{align}
    \left\langle w_{j, l, r}^{(t+1)}, \Mb_{j}^{r}\right\rangle&= \left\langle w_{j, l, r}^{(t)}, 
    \Mb_{j}^{r}\right\rangle+\frac{\eta}{n_s}\left[\left(1-O(\frac{1}{\polylog{K}})\right)\sum_{(\Xb,y)\in \mathcal{D}_{s}}\mathbb{I}\{y=j\}\pm O(\frac{s}{K^2})\right]\nonumber\\
    &\cdot \left([\langle w_{j,l,r}^{(t)}, \Mb_j^{r}\rangle^{+}]^{q-1}(z^{r}_{j})^{q}/\beta^{q-1}\right)\pm\widetilde{o}(\eta\sigma_0/K) \label{m1}
\end{align}
Let $l^{*}=\arg\max_{l} [\langle w^{(0)}_{j,l,r_{j}}, \Mb^{r_{j}}_{j}\rangle]^{+}$, and $l^{\prime}$ be arbitrary $l\in[m]$  Define:
$$
a_t= \langle w^{(t)}_{j,l^{*}, r_{j}}, \Mb_{j}^{r_{j}}\rangle,\quad b_t= \max\{\langle w^{(t)}_{j,l^{\prime}, 3-r_j}, \Mb^{3-r_{j}}_{j}\rangle, \sigma_0\} 
$$

By $(\ref{m1})$, we have $a_{t+1}\geq a_{t}+ A_t a_t^{q-1} $, $b_{t+1}\leq b_t+B_t b_{t}^{q-1}  $, where $A_t = \eta d_{j,r_{j}}(\mathcal{D})(1-O(\frac{1}{\polylog{K}}))$, $B_t = A_t M$, and $M=(1+\frac{1}{\polylog{K}})\cdot\frac{d_{j,3-r_{j}}(\mathcal{D})}{d_{j,r_{j}}(\mathcal{D})}$ is a constant.

Since $(j,r_{j})\in \mathcal{W}$, by definition we have $a_0\geq b_0 M^{\frac{1}{q-2}} (1+\frac{1}{\polylog{K}})$. Applying Lemma~\ref{tensor}, we can conclude that, once $a_t$ reaches $\widetilde{\Omega}(1)$ at some iteration after $T_j$, we still have $\Gamma^{(t)}_{3-r_{j}}\leq b_t\leq \widetilde{O}(a_0)=\widetilde{O}(\sigma_0)$.
\item Phase 2, Stage 1: for $t\in [T_j, T_0]$, let us denote $r^{\prime}=3-r_{j}$,  by  hypothesis that $\Gamma^{(t)}_{r^{\prime}}\leq \widetilde{O}(\sigma_0)$
\begin{enumerate}
    \item For $j\in\mathcal{S}^r$, or $(\Xb,y)\in \mathcal{D}_s$ and $j=y$ , we have 
    $$
    \sigma^{\prime}(\langle w_{j,l, r^{\prime}}^{(t)}, \Xb^{r^{\prime}}\rangle)z^{r^{\prime}}_{j}\leq \sigma^{\prime}(\langle w_{j,l, r^{\prime}}^{(t)}, M^{r^{\prime}}\rangle z^{r^{\prime}}_{j}\pm\widetilde{o}(\sigma_0))z^{r^{\prime}}_{j}\leq\widetilde{O}(\sigma_0^{q-1})
    $$
    \item For $(\Xb,y)\in \mathcal{D}_i$ and $j=y$, by induction hypothesis, we have: $\sigma^{\prime}(\langle w_{j,l, r^{\prime}}^{(t)}, \Xb^{r^{\prime}}\rangle)z^{r^{\prime}}_{j}\leq \widetilde{O}(\sigma_0^{q-1})$
\end{enumerate}
Putting back to $(\ref{gu})$, we obtain:
\begin{align*}
   &|\langle w_{j, l, r^{\prime}}^{(t+1)}, M_{j}^{r^{\prime}}\rangle|\leq |\langle w_{j, l, r^{\prime}}^{(t)}, M_{j}^{r^{\prime}}\rangle|\\
&+ \frac{\eta}{n_s}\sum_{(\Xb,y)\in \mathcal{D}_s}\left[\mathbb{I}\{y=j\}(\widetilde{O}(\sigma^{q-1}_0)+O(\sigma_g))\left(1-\ell_{j}\left(f^{(t)}, \Xb\right)\right) \right.+\mathbb{I}\{y\neq j\}\widetilde{O}(\sigma^{q-1}_0)\left.\ell_{j}\left(f^{(t)}, \Xb\right)\right]\\
&+\widetilde{O}(\frac{\sigma_0^{q-1}n_i}{n}) \cdot \frac{\eta}{n_i}\sum_{(\Xb,y)\in \mathcal{D}_i}\left[(\mathbb{I}\{y=j\}+\frac{1}{K}\mathbb{I}\{y\neq j\})\left(1-\ell_{y}\left(f^{(t)}, \Xb\right)\right) \right]
\end{align*}
 In this stage, we ignore the insufficient multi-modal data. Then, we have

\begin{align*}
    |\langle w_{j, l, r^{\prime}}^{(t+1)}, M_{j}^{r^{\prime}}\rangle|&\leq |\langle w_{j, l, r^{\prime}}^{(T_j)}, M_{j}^{r^{\prime}}\rangle|+\eta \widetilde{O}(\sigma_0^{q-1}) (Err^{\text{Tol, Stage 2 }}_{s,j}+T_0\cdot \widetilde{Err}^{\text{Stage 2}}_{s,j})\\
    &\leq \widetilde{O}(\sigma_0)+  \widetilde{O}(\sigma_0^{q-1})\cdot(\widetilde{O}(1)+O(\frac{1+s\Lambda}{\sigma^{q-2}_0}))(\sigma_0)\text{ (applying Claim~\ref{cla-2.1}) }\\
    &=  \widetilde{O}(\sigma_0)
 \end{align*}
\item Phase 2, Stage 2: for $t\geq T_0$, denote:
$$
Err^{\text{Tol, Stage 3 }}_{s}:=\sum_{t\geq T_0}\frac{1}{n_s}\sum_{(\Xb,y)\in \mathcal{D}_s} \left(1-\ell_{y}\left(f^{(t)}, \Xb\right)\right) 
$$
$$
Err^{\text{Tol, Stage 3 }}_{in,j}:=\sum_{t\geq T_0}\frac{1}{n_i}\sum_{(\Xb,y)\in \mathcal{D}_i} \mathbb{I}\{y=j\} \left(1-\ell_{y}\left(f^{(t)}, \Xb\right)\right) 
$$
Taking the insufficient multi-modal data into consideration, we have:
\begin{align*}
    \Gamma_{j, r^{\prime}}^{(t+1)} &\leq \Gamma_{j, r^{\prime}}^{(T_0)}+\widetilde{O}\left(\eta \sigma_{0}^{q-1}\right)Err^{\text{Tol, Stage 3 }}_{s}+O\left(\frac{\eta n_{i}}{n}\right) \cdot\left(Err^{\text{Tol, Stage 3 }}_{in,j}+\frac{\sum_{i \in[K]} Err^{\text{Tol, Stage 3 }}_{in,i}}{K}\right) \cdot \widetilde{O}\left(\sigma_{0}^{q-1}\right)  \\
    &\leq \widetilde{O}(\sigma_0)+   \widetilde{O}\left(\eta \sigma_{0}^{q-1}\right)\cdot (O(\frac{K}{\eta})+\widetilde{O}\left(\frac{n_{i} s}{\eta K \gamma^{q-1}}\right)+ \frac{n_i}{n}\cdot \widetilde{O}\left(\frac{n}{\eta K\gamma^{q-1}}\right))\\ &\text{(Applying Claim~\ref{cla-inerr} (b) and Claim~\ref{cla-2.2})}
\end{align*}
If $n_i\leq \frac{\gamma^{q-1}K^2}{s},\quad n_i\leq \frac{\gamma^{q-1}K}{\sigma_0^{q-2}}$ (already satisfied in our parameter settings), we can complete the proof.

\end{itemize}

\end{proof}
\subsection{Regularization}
\begin{lemma}[Diagonal Correlations]\label{lemma-dia} Suppose Induction Hypothesis holds for all iterations $<t$. 
 Then, letting $\Phi_{j, r}^{(t)} \stackrel{\text { def }}{=} \sum_{l \in[m]}\left[\left\langle w_{j, l, r}^{(t)}, \Mb^{r}_{j}\right\rangle\right]^{+}$, we have
$$
\forall j \in[K], \forall r \in[2]: \quad \Phi_{j, r}^{(t)} \leq \widetilde{O}(1)
$$
This implies $\Gamma_{j}^{(t)} \leq \widetilde{O}(1)$ as well.
\end{lemma}
\begin{proof}
By gradient updates, we have:
$$
\left[\left\langle w_{j, l, r}^{(t+1)}, \Mb_{j}^{r}\right\rangle\right]^{+}=\left[\left\langle w_{j, l, r}^{(t)}, \Mb_{j}^{r}\right\rangle\right]^{+}+ \theta_{j, l, r}^{(t)}\cdot \frac{\eta}{n}\underset{(\Xb, y) \in \mathcal{D}}{\sum}\left[\left\langle-\nabla_{w_{j, l, r}} L\left(f^{(t)} ; \Xb, y\right), \Mb_{j}^{r}\right\rangle\right]
$$
where $\theta_{j, l, r}^{(t)}\in[0,1]$. Considering the insufficient multi-modal data with label $y=j$ that modality $\mathcal{M}_r$ is insufficient, denoted by $\mathcal{I}_{j,r}$, we can define:
$$
I_{j,r}^{(t+1)}:= I_{j,r}^{(t)}+ \frac{\eta}{n}\sum_{l\in [m]}\theta_{j, l, r}^{(t)} \underset{(\Xb, y) \in \mathcal{I}_{j,r}}{\sum}\left[\left\langle-\nabla_{w_{j, l, r}} L\left(f^{(t)} ; \Xb, y\right), \Mb_{j}^{r}\right\rangle\right],\quad I_{j,r}^{(0)}=0
$$
$$
S_{j,r}^{(t+1)}:= S_{j,r}^{(t)}+ \frac{\eta}{n}\sum_{l\in [m]} \theta_{j, l, r}^{(t)} \underset{(\Xb, y) \notin \mathcal{I}_{j,r}}{\sum}\left[\left\langle-\nabla_{w_{j, l, r}} L\left(f^{(t)} ; \Xb, y\right), \Mb_{j}^{r}\right\rangle\right],\quad S_{j,r}^{(0)}= \Phi_{j, r}^{(0)}
$$
$$
\Phi_{j, r}^{(t)} =I^{(t)}_{j,r}+S^{(t)}_{j,r}
$$
For $I^{(t)}_{j,r}:$
$$
I_{j,r}^{(t+1)}:= I_{j,r}^{(t)}+ \frac{\eta}{n}\sum_{l\in [m]}\theta_{j, l, r}^{(t)} \underset{(\Xb, y) \in \mathcal{I}_{j,r}}{\sum}\left[(1-\ell_j\left(f^{(t)},\Xb\right))(\sigma^{\prime}(\langle w_{j,l, r}, \Xb^{r}\rangle)z^{r}_{j}\pm O(\sigma_g)  )\right]
$$

Since $\mathcal{M}_r$ is insufficient, $z_{j}^{r}\leq O(\gamma)$,  and we can easily conclude that, 
$$
|I_{j,r}^{(t+1)}-I_{j,r}^{(t)}|\leq O(\frac{\eta n_i\gamma}{n}) \sum_{l\in [m]}\frac{1}{n_i} \underset{(\Xb, y) \in \mathcal{D}_{i}}{\sum}\left[\mathbb{I}\{(\Xb, y) \in \mathcal{I}_{j,r}\}(1-\ell_j\left(F^{(t)},X\right))(\sigma^{\prime}(\langle w_{j,l, r}, X^{r}\rangle)\pm O(\sigma_g)  )\right]
$$
Denote:
$$
\hat{Err}^{\text{Tol, Stage 3 }}_{in}:=\sum_{t\geq T_0}\frac{1}{n_i}\sum_{(\Xb,y)\in \mathcal{D}_i}  \left(1-\ell_{y}\left(f^{(t)}, \Xb\right)\right) \sigma^{\prime}(\langle w_{j,l, r}, \Xb^{r}\rangle)
$$

Then we have, $\forall t\geq 0:$

$$
|I_{j,r}^{(t)}|\leq \widetilde{O}(\frac{\eta \gamma n_i}{Kn})(\hat{Err}^{\text{Tol, Stage 3 }}_{in}+T_0) =\widetilde{O}(\frac{\gamma n_i}{K})\leq \frac{1}{\operatorname{polylog}{(K)}}\quad \text{(Applying Claim~\ref{cla-inerr} (a) )}
$$
Hence, we only need to bound the remaining part $S_{j,r}^{(t)}$. Also by gradient inequality, we have:
\begin{align*}
    S_{j,r}^{(t+1)}\leq S_{j,r}^{(t)}+ O(\frac{\eta}{n}) \underset{(\Xb, y) \notin \mathcal{I}_{j,r}}{\sum}\left[\mathbb{I}\{y=j\}(1-\ell_y(f^{(t)},\Xb)\right]+\widetilde{O}(\eta\sigma_g) 
\end{align*}
Let us denote: $\Phi^{(t) \stackrel{\text { def }}{=}} \max _{j\in[K], r\in[2]} \Phi_{j,r }^{(t)}$, and $(j^{*},r^{*})=\arg\max S_{j,r}^{(t)}$. For $t\geq T_0$,  if $S_{j^{*},r^{*}}^{(t)}>\polylog{(K)}$, then we obtain $\Phi^{(t)}>\polylog{(K)}$. For $(\Xb,y)\in \mathcal{D}_s$ with $y=j^{*}$; and for $(\Xb,y)\in \mathcal{D}_i$ with $y=j^{*}$ and $\mathcal{M}_{r^{*}}$ is sufficient we both have:
\begin{itemize}
    \item $f^{(t)}_{j}(\Xb)\leq (c_1+c_2 +o(1))\Phi^{(t)},\quad j\neq j^{*}$
    \item $f^{(t)}_{j^{*}}(\Xb)\geq (1-o(1))\Phi^{(t)}$ 
\end{itemize}

Hence $1-\ell_{j^{*}}(f^{(t)},\Xb)=\exp(-\Omega(\polylog{(K)}))$ is neglibible. Then $$\max S_{j,r}^{(t+1)}\leq S_{j,r}^{(t)}+\widetilde{O}(\eta(\exp(-\Omega(\polylog{(K)}))+\sigma_g))=\widetilde{O}(1)$$
Thus, we complete the proof.
\end{proof}
\begin{lemma}[Nearly Non-Negative]\label{lemma-non} Suppose Induction Hypothesis holds for all iterations $<t$. Then,
    $$
    \forall j \in[K], \forall l \in[m], \forall r \in[2]: \quad\left\langle w_{j, l, r}^{(t)}, M_{j, r}\right\rangle \geq-\widetilde{O}\left(\sigma_{0}\right)
    $$
    \end{lemma}
    \begin{proof}
    By gradient updates, we obtain:
    \begin{align*}
    \left\langle w_{j, l, r}^{(t+1)}, \Mb_{j}^{r}\right\rangle&\geq \left\langle w_{j, l, r}^{(t)}, \Mb_{j}^{r}\right\rangle\\
    &+\frac{\eta}{n}\sum_{(\Xb,y)\in \mathcal{S}}\left[\mathbb{I}\{y=j\}\left(1-\ell_{j}\left(f^{(t)}, \Xb\right)\right)\left(\sigma^{\prime}(\langle w_{j,l, r}^{(t)}, \Xb^{r}\rangle)z^{r}_{j}-O\left(\sigma_{g}\right)\right)\right.\\
    &-\mathbb{I}\{y\neq j\}\left.\ell_{j}\left(f^{(t)}, \Xb\right)\left(
    \mathbb{I}\{j\in \mathcal{S}^{r}\}\sigma^{\prime}(\langle w_{j,l, r}^{(t)}, \Xb^{r}\rangle)z^{r}_{j}+\widetilde{O}(\sigma^{q-1}_0)\alpha+   O\left(\sigma_{g}\right)\right)\right]
    \end{align*}
   For $y=j$, we have $\sigma^{\prime}(\langle w_{j,l,r}^{(t)}, \Xb^{r}\rangle)z^{r}_{j}\geq 0$.  If there exists $t_0$, s.t.  $\langle w_{j, l, r}^{(t)}, \Mb_{j, r} \rangle \leq -\widetilde{\Omega}\left(\sigma_{0}\right)$ for $t\geq t_0$, then for $j\in \mathcal{S}^{r}$, we have $
    \sigma^{\prime}(\langle w_{j,l, r}^{(t)}, \Xb^{r}\rangle)z^{r}_{j} =  \sigma^{\prime}(\langle w_{j,l, r}^{(t)}, \Mb_j^{r}\rangle z^{r}_{j}\pm \widetilde{o}(\sigma_0))z^{r}_{j}=0 $. Therefore,
    \begin{align*}
    &\left\langle w_{j, l, r}^{(t+1)}, \Mb_{j}^{r}\right\rangle\geq \left\langle w_{j, l, r}^{(t)}, \Mb_{j}^{r}\right\rangle\\&-\frac{\eta}{n}\sum_{(\Xb,y)\in \mathcal{S}}\left[\mathbb{I}\{y=j\}\left(1-\ell_{j}\left(f^{(t)}, \Xb\right)\right)O\left(\sigma_{g}\right)\right.+\mathbb{I}\{y\neq j\}\left.\ell_{j}\left(f^{(t)}, \Xb\right)\left(\sigma_0^{q-1} \alpha+O\left(\sigma_{g}\right)\right)\right]
    \end{align*}
    First consider the case $t\leq T_0=\Theta(\frac{K}{\eta\sigma_0^{q-2}})$, we have $\ell_{j}(f^{(t)},\Xb)=O(1/K)$
    , hence 
    \begin{align*}
        \left\langle w_{j, l, r}^{(t+1)}, \Mb_{j}^{r}\right\rangle\geq -\widetilde{O}\left(\sigma_{0}\right)-O(\frac{\eta T_0(\sigma_g+\sigma_0^{q-1}\alpha)}{K})=-\widetilde{O}\left(\sigma_{0}\right)
    \end{align*}
  $\sigma_g=O(\sigma^{q-1}_0)$.  
    When $t\geq T_0$, notice that for $\mathcal{D}_i$, $\ell_j(f^{(t)},\Xb)=O(\frac{1}{K})(1-\ell_y(f^{(t)},\Xb))$ when $j\neq y$ (by Fact~\ref{fact-err}), then we have:
    \begin{align*}
    &\left\langle w_{j, l, r}^{(t+1)}, \Mb_{j}^{r}\right\rangle\geq \left\langle w_{j, l, r}^{(t)}, \Mb_{j}^{r}\right\rangle\\
    &-\frac{\eta }{n_s}\sum_{(\Xb,y)\in \mathcal{D}_{s}}\left[\left(1-\ell_{y}(f^{(t)}, \Xb)\right)\left(\sigma_0^{q-1}\alpha+O\left(\sigma_{g}\right)\right)\right]\\
    &-\frac{\eta n_i}{n}\cdot \frac{1}{n_i}\sum_{(\Xb,y)\in \mathcal{D}_{i}}\left[\mathbb{I}\{y=j\}\left(1-\ell_{y}\left(f^{(t)}, \Xb\right)\right)O\left(\sigma_{g}\right)+ \mathbb{I}\{y\neq j\}\left(1-\ell_{y}(f^{(t)}, \Xb)\right)\frac{\sigma_0^{q-1}\alpha+O\left(\sigma_{g}\right)}{K}\right]
    \end{align*}
    we need to bound:
    \begin{align*}
        Err^{\text{Tol, Stage 3 }}_{s}\leq \widetilde{O}(\frac{1}{\eta\sigma_0^{q-2}})\\
        Err^{\text{Tol, Stage 3 }}_{in,j}\cdot \frac{\eta n_i}{n} \leq  \widetilde{O}(\frac{1}{\sigma^{q-2}_0})
    \end{align*}
    Combining the results from Claim~\ref{cla-2.2} and~\ref{cla-inerr}, we c complete the proof.
    \end{proof}
    \begin{lemma}[Off-Diagnol Correlation] \label{lemma-off}Suppose Induction Hypothesis holds for all iterations $<t$. Then,
        $$
        \forall j \in [K], \forall l \in[m], \forall i \in[K] \backslash\{j\}: \quad\left|\left\langle w_{j, l, r}^{(t)}, \Mb_{i}^{r}\right\rangle\right| \leq \widetilde{O}\left(\sigma_{0}\right)
        $$
        \end{lemma}
        \begin{proof}
        Denote $A^{t}_{j}=\max_{l\in[m], i\in [K]/j} \left|\left\langle w_{j, l, r}^{(t)}, \Mb_{i}^{r}\right\rangle\right| $. By gradient inequality, we have:
        \begin{align*}
        \left|\left\langle w_{j, l, r}^{(t+1)}, \Mb_{i}^{r}\right\rangle\right|&\leq \left|\left\langle w_{j, l, r}^{(t)}, \Mb_{i}^{r}\right\rangle\right|\\
        &+\frac{\eta}{n}\sum_{(\Xb,y)\in \mathcal{D}}\left[\mathbb{I}\{y=j\}\left(1-\ell_{j}\left(f^{(t)}, \Xb\right)\right)\left(\sigma^{\prime}(\langle w_{j,l, r}^{(t)}, \Xb^{r}\rangle)(\mathbb{I}\{i \in  \mathcal{S}^{r} \}z^{r}_{i}+\alpha^{r}_{i})
        +O\left(\sigma_{g}\right)\right)\right.\\
        &+\mathbb{I}\{y\neq j\}\left.\ell_{j}\left(f^{(t)}, \Xb\right)\left(\sigma^{\prime}(\langle w_{j,l, r}^{(t)}, \Xb^{r}\rangle)\left(z^{r}_{i}\mathbb{I}\{i=y, \text{or } i\in \mathcal{S}^{r}(\Xb) \}+\alpha^{r}_{i}\mathbb{I}\{i\neq y\}\right)+\widetilde{O}\left(\sigma_{g}\right)\right)\right]
        \end{align*}
        \begin{itemize}
            \item Phase 1: $t\in [0, T_j]$. We have $\ell_{j}(f^{(t)},\Xb)\leq O(\frac{1}{K})$
            \begin{align*}
        \left|\left\langle w_{j, l}^{(t+1)}, \Mb_{i}^{r}\right\rangle\right|&\leq \left|\left\langle w_{j, l}^{(t)}, \Mb_{i}^{r}\right\rangle\right|+\widetilde{O}(\frac{\eta}{K})\left( (\Gamma^{(t)}_{j})^{q-1}\cdot(\alpha+ \frac{s}{K})+O(\sigma_{g}) \right)
        \end{align*}
        Combining with the growth rate $\frac{\eta}{K}\sum_{t\leq T_{j} }(\Gamma^{(t)}_{j})^{q-1}\leq \widetilde{O}(1)$, and $T_{j}\leq \Theta(\frac{K}{\eta\sigma^{q-2}_0})$
        , as long as
        $$
       \frac{s}{K} =\widetilde{O}(\sigma_0),\quad \alpha=\widetilde{O}(\sigma_0)
        ,\quad \sigma_g = \widetilde{O}(\sigma^{q-1}_0) 
        $$
        we have $A^{(t)}_{j}\leq \widetilde{O}(\sigma_0)$
        \item Phase 2, Stage 1: $t\in [T_j, T_0]$, when $y=j$, we naively bound the $\sigma^{\prime}(\langle w_{j,l, r}^{(t)}, \Xb^{r}\rangle)$ by $1$; for $j\neq y$, we write 
        $$
        \sigma^{\prime}(\langle w_{j,l, r}^{(t)}, \Xb^{r}\rangle)\leq \mathbb{I}\{j\in \mathcal{S}^{r}\}+\widetilde{O}(\sigma_0^{q-1})
        $$
        Then we have
        \begin{align*}
        \left|\left\langle w_{j, l, r}^{(t+1)}, \Mb_{i}^{r}\right\rangle\right|&\leq \left|\left\langle w_{j, l, r}^{(t)}, \Mb_{i}^{r}\right\rangle\right|\\
        &+\frac{\eta}{n_s}\sum_{(\Xb,y)\in \mathcal{D}_s}\left[\mathbb{I}\{y=j\}\left(1-\ell_{y}\left(f^{(t)}, \Xb\right)\right)\left(\frac{s}{K}+\alpha
        +O\left(\sigma_{g}\right)\right)\right.\\
        &+\mathbb{I}\{y\neq j\}\left.\ell_{j}\left(f^{(t)}, \Xb\right)\left(\mathbb{I}\{j\in \mathcal{S}^{r}\}\mathbb{I}\{i=y, \text{or } i\in \mathcal{S}^{r} \} O(1) +\widetilde{O}(\sigma_0^{q-1})\left(z^{r}_{i}+\alpha\right)+O(\sigma_{g})\right)\right]\\
        &+\frac{\eta n_i}{n}\cdot \frac{1}{n_i}\sum_{(\Xb,y)\in \mathcal{D}_i}\left(1-\ell_{y}\left(f^{(t)}, \Xb\right)\right)\left[\mathbb{I}\{y=j\}\left(\frac{s}{K}+\alpha
        +O\left(\sigma_{g}\right)\right)\right.\\
        &+\mathbb{I}\{y\neq j\}\cdot \frac{1}{K}\left.\left(\mathbb{I}\{j\in \mathcal{S}^{r}\}\mathbb{I}\{i=y, \text{or } i\in \mathcal{S}^{r}(X) \} O(1) +\widetilde{O}(\sigma_0^{q-1})\left(z^{r}_{i}+\alpha\right)+O(\sigma_{g})\right)\right]
        \end{align*}
        Hence, we need to bound:
        \begin{align*}
            Err^{\text{Tol, Stage 2 }}_{s,j}:= \sum_{t=T_j}^{T_0} \frac{1}{n_s}\sum_{(\Xb,y)\in \mathcal{D}_s}\left[\mathbb{I}\{y=j\}\left(1-\ell_{j}\left(f^{(t)}, \Xb\right)\right)\right]\leq 
           \widetilde{O}(\frac{1}{\eta})\\
           \widetilde{Err}^{\text{Stage 2}}_{s,j}:=\frac{1}{n_s}\sum_{(\Xb,y)\in \mathcal{D}_s}\mathbb{I}\{y\neq j\}\ell_{j}\left(f^{(t)}, \Xb\right)\leq O(\frac{1}{K})
        \end{align*}
        which can be directly implied from Claim~\ref{cla-2.1}.
        \item Phase 2, Stage 2: $t> T_0:$
        \begin{align*}
        \left|\left\langle w_{j, l, r}^{(t+1)}, \Mb_{i}^{r}\right\rangle\right|&\leq \left|\left\langle w_{j, l, r}^{(t)}, \Mb_{i}^{r}\right\rangle\right|\\
        &+\frac{\eta}{n_s}\sum_{(\Xb,y)\in \mathcal{D}_s}\left[\left(1-\ell_{y}\left(f^{(t)}, \Xb\right)\right)\left(O(\frac{s^2}{K^2})+\widetilde{O}(\sigma_0^{q-1})+\frac{\alpha}{K}
        +O\left(\sigma_{g}\right)\right)\right]\\
        &+\frac{\eta n_i}{n}\cdot \frac{1}{n_i}\sum_{(\Xb,y)\in \mathcal{D}_i}\left[\mathbb{I}\{y=j\}\left(\frac{s}{K}+\alpha
        +O\left(\sigma_{g}\right)\right)\right.\\
        &+\mathbb{I}\{y\neq j\}\cdot \frac{1}{K}\left.\left(O(\frac{s^2}{K^2})+\widetilde{O}(\sigma_0^{q-1})+O\left(\sigma_{g}\right)\right)\right](1-\ell_{y}(F^{(t)},\Xb))
        \end{align*}
        By the error analysis in Claim~\ref{cla-2.2} and~\ref{cla-inerr}, we have
        \begin{align*}
            Err^{\text{Tol, Stage 3 }}_{s}:=\sum_{t>T_0} \frac{1}{n_s}\sum_{(\Xb,y)\in \mathcal{D}_s}\left[\left(1-\ell_{y}\left(f^{(t)}, \Xb\right)\right)\right]\leq \widetilde{O}(\frac{K}{\eta})\\
            Err^{\text{Tol, Stage 3 }}_{in}:=\sum_{t>T_0}\frac{1}{n_i}\sum_{(\Xb,y)\in \mathcal{D}_i}(1-\ell_{y}\left(f^{(t)}, \Xb\right))\leq \widetilde{O}(\frac{n}{\eta \gamma^{q-1}})
        \end{align*}
        If $\frac{n_i}{\gamma^{q-1}K}\leq \widetilde{O}(\frac{1}{\sigma_0^{q-2}})$, then we completes the proof.
        \end{itemize}
        \end{proof}

        \begin{lemma}[Gaussian Noise Correlation]\label{lemma-gau}
            Suppose Induction Hypothesis holds for all iterations $<t$. Then,
            \begin{itemize}
                \item For $(\Xb,y)\in \mathcal{D}$, $j\notin \{y\}\cup \mathcal{S}^{r}(\Xb):$ $|\langle w_{j,l, r}^{(t)}, {\xi^{r}}^{\prime}\rangle|\leq \widetilde{O}(\sigma_0)$  
                \item For $(\Xb,y)\in \mathcal{D}$, $j\in\mathcal{S}^{r}(\Xb)$; or  $(\Xb,y)\in \mathcal{D}_s$, $j=y:$ $\langle w_{j,l, r}^{(t)}, {\xi^{r}}^{\prime}\rangle\leq \widetilde{o}(\sigma_0)$  
                \item For $(\Xb,y)\in \mathcal{D}_i$, $j=y$ and $(j,3-r)\in\mathcal{W}:$ $\langle w_{j,l, r}^{(t)}, {\xi^{r}}^{\prime}\rangle\leq \widetilde{O}(\sigma_0)$
            \end{itemize}
            
            \end{lemma}
            \begin{proof}
            By gradient updates, for $(\Xb_0,y_0)\in S$
            $$
            \begin{aligned}
            \left\langle w_{j, l, r}^{(t+1)}, {\xi^{r}_{0}}^{\prime}\right\rangle&=\left\langle w_{j, l, r}^{(t)},  {\xi^{r}_{0}}^{\prime}\right\rangle\\&+\frac{\eta}{n} \sum_{(\Xb, y)\in \mathcal{D}}\left[\mathbb{I}\{y=j\} \sigma^{\prime}\left(\left\langle w_{j, l, r}^{(t)}, \Xb^{r}\right\rangle\right)\left\langle \Xb^{r},  {\xi^{r}_{0}}^{\prime}\right\rangle\left(1-\ell_{j}\left(f^{(t)}, \Xb\right)\right)\right.\\
            &-\mathbb{I}\{y \neq j\} \left.\sigma^{\prime}\left(\left\langle w_{j, l, r}^{(t)}, \Xb^{r}\right\rangle\right)\left\langle X^{r},  {\xi^{r}_{0}}^{\prime}\right\rangle\ell_{j}\left(f^{(t)}, \Xb\right)\right]
            \end{aligned}
            $$
            If $j=y_0$, $|\langle \Xb^{r},  {\xi^{r}_{0}}^{\prime}\rangle|\leq \widetilde{O}(\sigma_g)=\widetilde{o}(\frac{1}{\sqrt{d}})$ except for $\Xb^{r}_0$, then we have:
            $$
            \begin{aligned}
            \left\langle w_{j, l, r}^{(t+1)}, {\xi^{r}}^{\prime}\right\rangle=\left\langle w_{j, l,r}^{(t)}, {\xi^{r}}^{\prime}\right\rangle\pm\frac{\eta}{\sqrt{d_r}}+\widetilde{\Theta}(\frac{\eta}{n})  \sigma^{\prime}\left(\left\langle w_{j, l, r}^{(t)}, \Xb^{r}\right\rangle\right)\left(1-\ell_{j}\left(f^{(t)}, \Xb\right)\right)
            \end{aligned}
            $$
            Else $j\neq y_0:$
            $$
            \begin{aligned}
            \left\langle w_{j, l, r}^{(t+1)}, {\xi^{r}}^{\prime}\right\rangle=\left\langle w_{j, l, r}^{(t)}, {\xi^{r}}^{\prime}\right\rangle\pm \frac{\eta}{\sqrt{d_r}}-\widetilde{\Theta}(\frac{\eta}{n})  \sigma^{\prime}\left(\left\langle w_{j, l, r}^{(t)}, \Xb^{r}\right\rangle\right)\ell_{j}\left(f^{(t)}, \Xb\right)
            \end{aligned}
            $$
            If $|\langle w_{j, l, r}^{(t)}, \Xb^{r}\rangle|\leq\widetilde{O}(c)$, hence $\sigma^{\prime}(\langle w_{j, l, r}^{(t)}, \Xb^{r}\rangle)\leq \widetilde{O}(c^{q-1})$. When $t\leq T_0$,
            $$
            \begin{aligned}
            |\langle w_{j, l, r}^{(t+1)}, {\xi^{r}}^{\prime}\rangle|\leq \frac{T_0\eta}{\sqrt{d}}+\widetilde{O}(\frac{\eta c^{q-1}T_0}{n}) 
            \end{aligned}
            $$
            \begin{itemize}
                \item  Sufficient: by Claim~\ref{cla-individual}
            \begin{align*}
                \sum_{t>T_0}\ell_{j}(f^{(t)},\Xb)\leq  \sum_{t>T_0}(1-\ell_{y}(f^{(t)},\Xb))\leq \widetilde{O}(\frac{K^3}{s^2})\sum_{t>T_0}\frac{1}{n_s}\sum_{(\Xb,y)\in \mathcal{D}_s}  (1-\ell_{y}(f^{(t)},\Xb))
            \end{align*}
            Combining the previous analysis:
            $$
            \begin{aligned}
            |\langle w_{j, l,r}^{(t+1)}, {\xi^{r}}^{\prime}\rangle|\leq \frac{T\eta}{\sqrt{d}}+\widetilde{O}(\frac{\eta c^{q-1}}{n}(T_0+\frac{K^4}{s^2\eta})) =\frac{T\eta}{\sqrt{d}}+\widetilde{O}(\frac{\eta c^{q-1}T_0}{n}+\frac{K^4c^{q-1}}{s^2n})
            \end{aligned}
            $$
            When $j\notin \{y\}\cup \mathcal{S}^{r}(X)$, $c=\widetilde{O}(\sigma_0)$; else, $c=\widetilde{O}(1)$. $n \geq \widetilde{\omega}\left(\frac{K}{\sigma_{0}^{q-1}}\right),n \geq \widetilde{\omega}\left(\frac{k^{4}}{s^2\sigma_{0}}\right), \frac{T}{\eta \sqrt{d}}\leq 1/\poly(K)$ 
            \item Insufficient: by Claim~\ref{cla-inerr}
            $$
            \sum_{t>T_0}(1-\ell_{y}(f^{(t)},\Xb))\leq \widetilde{O}(\frac{n}{\eta \gamma^{q-1}})
            $$
            Similarly, we have:
            $$
            \begin{aligned}
            |\langle w_{j, l, r}^{(t+1)}, {\xi^{r}}^{\prime}\rangle|\leq \frac{T\eta}{\sqrt{d_r}}+\widetilde{O}(\frac{\eta c^{q-1}}{n}(T_0+\frac{n}{\eta \gamma^{q-1}})) 
            \end{aligned}
            $$
            For $j\neq y\cup \mathcal{S}^{r}$ or $r\notin\mathcal{W}$, $c=\widetilde{O}(\sigma_0)$. $\sqrt{d_r} \geq \eta T \cdot \operatorname{poly}(K)$, $\sigma_0^{q-2}\leq \gamma^{q-1}$.
            \end{itemize}
            
            \end{proof}
            \subsection{Proof for Induction Hypothesis~\ref{hypo}}
             Now we are ready to prove the Induction Hypothesis~\ref{hypo}. We frist restate the following theorem:
           \begin{theorem}
               Under the global parameter settings in~\ref{sec-not}, for $\eta\leq \frac{1}{\poly(K)}$, and sufficiently large $K$, Induction Hypothesis~\ref{hypo} holds for all iteration $t\leq T$.
           \end{theorem}
           \begin{proof}
        At iteration $t$, it is easy to derive that:
      \begin{align}\label{equ}
        \left\langle w_{j,l,r}^{(t)}, \Xb^{r}\right\rangle=\sum_{i \in \{y\}\cup\mathcal{S}^{r}}\left\langle w_{j,l, r}^{(t)}, \Mb^{r}_{i}\right\rangle z^{r}_{i}+\sum_{i \in [K]} \alpha^{r}_{i}\left\langle w_{j,l,r}^{(t)}, \Mb_{j}^{r}\right\rangle+\left\langle w_{j,l, r}^{(t)}, {\xi^{r}}^{\prime}\right\rangle
      \end{align}
      It is easy to verify the statements hold at $t=0$ using standard Gausian analysis. Suppose it holds for iterations $<t$, combining the lemmas we have established, we can have:
      \begin{itemize}
        \item [(a).]$\left\langle w_{j,l, r}^{(t)}, \Mb^{r}_{3-r_j}\right\rangle\leq \widetilde{O}(\sigma_0)$, for every $l\in[m]$, where $\mathcal{M}_{r_j}$ is the winning modality for class $j$. (By Lemma~\ref{lemma-comp})
        \item [(b).]$\left\langle w_{j,l, r}^{(t)}, \Mb^{r}_{j}\right\rangle\in[- \widetilde{O}(\sigma_0), \widetilde{O}(1)]$  for every $j\in[K]$. $r\in[2]$, $l\in[m]$ (By Lemma~\ref{lemma-dia} and ~\ref{lemma-non})
          \item [(c).]$\left\langle w_{j,l, r}^{(t)}, \Mb^{r}_{i}\right\rangle\leq \widetilde{O}(\sigma_0)$ for $i\neq j$, for every $j\in[K]$. $r\in[2]$, $l\in[m]$ (By Lemma~\ref{lemma-off})
      \end{itemize}
      Induction Hypothesis~\ref{hypo} \romannum{6},   \romannum{7} have been proven by above results.
\begin{itemize}
    \item For Induction Hypothesis~\ref{hypo}\romannum{1}, plug $(b)$ and $(c)$ into $(\ref{equ})$ and applying $\langle w_{j,l, r}^{(t)}, {\xi^{r}}^{\prime}\rangle\leq \widetilde{o}(\sigma_0)$ in Claim~\ref{lemma-gau}  ;
    \item For Induction Hypothesis~\ref{hypo}\romannum{2}, plug $(c)$ into $(\ref{equ})$ and applying $\langle w_{j,l, r}^{(t)}, {\xi^{r}}^{\prime}\rangle\leq \widetilde{O}(\sigma_0)$ in Claim~\ref{lemma-gau}
    \item For Induction Hypothesis~\ref{hypo}\romannum{3}, plug $(b)$ and $(c)$ into $(\ref{equ})$ and use $\alpha^{r}_{i}\in[0,\alpha]$.
   \item For Induction Hypothesis~\ref{hypo}\romannum{4}, plug $(b)$ and $(c)$ into $(\ref{equ})$ and applying $\langle w_{j,l, r}^{(t)}, {\xi^{r}}^{\prime}\rangle\leq \widetilde{o}(\sigma_0)$ in Claim~\ref{lemma-gau}  ;
      \item For Induction Hypothesis~\ref{hypo}\romannum{5}, plug $(a)$ and $(c)$ into $(\ref{equ})$ and applying $\langle w_{j,l, r}^{(t)}, {\xi^{r}}^{\prime}\rangle\leq \widetilde{O}(\sigma_0)$ in Claim~\ref{lemma-gau}  ;
\end{itemize}
  Therefore, we completes the proof.
           \end{proof}

        \subsection{Main Theorems for Multi-mdoal}
        \begin{theorem}[Theorem~\ref{thm-mul} Restated]\label{thm-main}
            For sufficiently large $K>0$, every $\eta \leq \frac{1}{\operatorname{poly}(K)}$, after $T=\frac{\text { poly }(k)}{\eta}$ many iteration, for the multi-modal network $f^{(t)}$, and ${f^{r}}^{(t)}:=\mathcal{C}(\varphi^{(t)}_{\mathcal{M}_r})$, \textit{w.h.p} : 
               \begin{itemize}
                   \item Training error is zero: 
                   $$
           \frac{1}{n}\sum_{(\Xb,y)\in\mathcal{D} }\mathbb{I}\{\exists j\neq y: f_{y}^{(T)}(\Xb) \leq f_{j}^{(T)}(\Xb)\}=0.
                   $$
                   \item For $r\in[2]$, with probability $p_{3-r}>0$, the test error of ${f^{r}}^{(T)}$ is high:
                   \begin{align*}
                       \Pr_{(\Xb^{r},y)\sim\mathcal{P}^{r}}( {f_{y}^{r}}^{(T)}(\Xb^r)\leq \max_{j\neq y} {f_{j}^{r}}^{(T)}(\Xb^{r})-\frac{1}{\polylog(K)} )\geq \frac{1}{K}
                   \end{align*}
                    where $p_1+p_2=1-o(1)$, and $p_r\geq m^{-O(1)}$, $\forall r\in[2]$.
                    
           
               
               \end{itemize}

                      \end{theorem}
 

      \begin{proof}

 \textbf{Training error analysis.} For every data pair $(\Xb,y)$: $\ell_y(f^{(t)},\Xb)\geq \frac{1}{2}\Rightarrow -\log(\ell_{y}(f^{(t)},\Xb))$ can be bounded by $O(1-\ell_y(f^{(t)},\Xb))$; On the other hand, we observe that  $\ell_y(f^{(t)},\Xb)$ cannot be smaller than $\frac{1}{2}$ for too many pairs in Phase 2, Stage 2, and in this case $-\log(\ell_{y}(f^{(t)},\Xb))$ can be naively bounded by $\widetilde{O}(1)$, since by Claim~\ref{cla-2.2} and ~\ref{cla-inerr}:
    $$
        \sum_{t=T_{0}}^{T}\left(1-\ell_{y}\left(f^{(t)}, \Xb\right)\right) \leq \widetilde{O}\left(\frac{n}{\eta \gamma^{q-1}}\right);\quad 
    Err^{\text{Tol, Stage 3 }}_{s}\leq \widetilde{O}\left(\frac{K}{\eta}\right)+\widetilde{O}\left(\frac{n_{i} s}{\eta K \gamma^{q-1}}\right) 
    $$
Therefore, we can bound the average training obejctive in Phase 2, Stage 2 as follows:
          $$\frac{1}{T}\sum^{T}_{t=T_0}\mathcal{L}(f^{(t)})=\frac{1}{T}\sum^{T}_{t=T_0}\frac{1}{n}\sum_{(\Xb,y)\in\mathcal{D}} -\log(\ell_y(f^{(t)},\Xb)) \leq \frac{1}{\poly(K)}$$ Combining with the non-increasing property of gradient descent algorithm acting on Lipscthiz continuous objective function, we obtain:
           $$
           \frac{1}{n}\sum_{(\Xb,y)\in\mathcal{D}} (1-\ell_y(f^{(T)},\Xb))\leq \frac{1}{n}\sum_{(\Xb,y)\in\mathcal{D}} -\log(\ell_y(f^{(T)},\Xb))\leq \frac{1}{\poly(K)}
           $$
           Therefore, we can conclude the training error is sufficiently small at the end of the iteration $T$.
           
           \textbf{Test error analysis.} For the test error of ${f^{r}}^{(T)}$, given $j\in[K]$, by Lemma~\ref{lemma-win}, with probability $p_{j,3-r}$ that $\mathcal{M}_{3-r}$ is the winning modality for class $j$. In this case, according to Lemma~\ref{lemma-comp}, $\Gamma^{(T)}_{j,r}\leq \widetilde{O}(\sigma_0)$. 

By Claim~\ref{cla-individual}, we have $c\Phi_j^{(T)}-\Phi_i^{(T)}\leq -\Omega(\log(K))$ for any $j,i\in[K]$, since $\frac{1}{n_s}\underset{(\Xb, y) \in\mathcal{D}_{s}}{\sum}\left[1-\ell_{y}\left(f^{(t)}, \Xb\right)\right]\leq \frac{1}{K^3}$. Hence $\Phi_j^{(T)}\geq \Omega(\log(K))$, and at least for the winning modality $\mathcal{M}_{r_j}$, $\Phi_{j,r_j}^{(T)}\geq \Omega(\log(K))$.

Now for $(\Xb^{r},y)\sim\mathcal{P}^{r}$, with $y=j$, by the function approximation in Fact~\ref{fact-app}, we have ${f_{y}^{r}}^{(T)}\leq \widetilde{O}(\sigma_0)+\frac{1}{\polylog(K)}$. For every other $i\neq y$, as long as $\mathcal{M}_{r}$ is the winning modality for class $i$ (which happens with probability $p_{i,r}$ for every $i$) and also belongs to $\mathcal{S}^{r}$, again using  Fact~\ref{fact-app} with $\mathcal{M}_{r}$, $\Phi_{i,r}^{(T)}\geq \Omega(\log(K))$, we have  ${f^{r}_{i}}^{T}\geq \widetilde{\Omega}(\rho_r)$. Such event occurs for some $i$ with probability  $\Omega(\frac{s}{K})$, and we can obtain:
$$
{f^{r}_{y}}^{(T)}(\Xb^{r})\leq \max_{i\neq y} {f^{r}_{i}}^{(T)}(\Xb^{r})-\frac{1}{\polylog(K)}
$$
Therefore, with probability $p_{r}=\sum_{j\in[K]}p_{j,r}$, the test error is high:
\begin{align*}
    \Pr_{(\Xb^{r},y)\sim\mathcal{P}^{r}}( {f_{y}^{r}}^{(T)}(\Xb^r)\leq \max_{j\neq y} {f_{j}^{r}}^{(T)}(\Xb^{r})-\frac{1}{\polylog(K)} )\geq \frac{1}{K}
\end{align*}



      \end{proof}
\begin{corollary}[Corollary~\ref{col} Restated]
        Suppose the assumptions in Theorem~\ref{thm-main} holds, \textit{w.h.p}, for joint training, the learned multi-modal network $f^{(T)}$ satisfies:
        \begin{align*}
            \Pr_{(\Xb,y)\sim\mathcal{P}}(\exists j\neq y:f_{y}^{(T)}(\Xb)\leq
            f_{j}^{(T)}(\Xb) )\in [\sum_{r\in[2]} (p_r-o(1))\mu_r, \sum_{r\in[2]} (p_r+o(1))\mu_r]
        \end{align*}
\end{corollary}
\begin{proof}
           \begin{itemize}
               \item   If $(\Xb, y)$ is sufficient,  following the similar analysis in Theorem~\ref{thm-main},
               we have $c\Phi_j^{(t)}-\Phi_i^{(t)}\leq -\Omega(\log(K))$ for any $j,i\in[K]$.  Applying the Fact~\ref{fact-app}, we conclude that $f_{y}^{(T)}(\Xb) \geq \max _{j \neq y} f_{j}^{(T)}(\Xb)+\Omega(\log k)$ \textit{w.h.p}.
\item If $(\Xb, y)$ is insufficient, by the choice of $\mu_r$, we only consider the case that at most one modality data $\Xb^{r}$ is insufficient. Consider $\mathcal{M}_r$ is insufficient, i.e. its sparse vector falls into the insufficient class. With probability $p_{j,r}$, $\mathcal{M}_{r}$ wins the competition, and we obtain  ${f_{y}^{r}}^{(T)}\leq O(\gamma_r)+\frac{1}{\polylog(K)}$. Moreover, combining with the fact that $\Phi_i^{(t)}\geq \Omega(\log(K))$, if some $j\in\mathcal{S}^{1}(X)\cup\mathcal{S}^{2}(X)$, we obtain $f_{j}^{(T)}(\Xb) \geq \widetilde{\Omega}(\rho_r)$, which happens with probability at least $1-e^{-\Omega\left(\log ^{2} k\right)}$. In this case, $
f_{y}^{(T)}(\Xb) \leq \max _{j \neq y} f_{j}^{(T)}(\Xb)-\frac{1}{\operatorname{polylog}(k)}
$
           \end{itemize}
By above arguments, the test error maily comes from the insufficeint data, and consequently $\Pr_{(\Xb,y)\sim\mathcal{P}}(\exists j\neq y:f_{y}^{(T)}(\Xb)\leq
f_{j}^{(T)}(\Xb) )$ is around $\sum_{r\in[2]}p_r\mu_r$.

\end{proof}

\section{Results for Uni-modal Networks}
In this section, we will provide the proof sketch of Theorem~\ref{thm-uni} for uni-modal networks. The proof follows the analyis of joint training ver closely, but it is easier since we do not need to consider the modality competition. Similarly, we first introduce the induction hypothesis for unimodal, and then utilize the it to prove the main results.
\subsection{Induction Hypothesis}
For each class $j\in [K]$, let us denote:
    $$
    \Psi_{j, r}^{(t)} \stackrel{\text { def }}{=} \max _{l \in[m]}\left[\left\langle \nu_{j, l, r}^{(t)}, \Mb_{j}^{r}\right\rangle\right]^{+} \quad 
    \Pi^{(t)}_{j,r}:= \sum_{l \in[m]}\left[\left\langle \nu_{j, l, r}^{(t)}, \mathbf{M}_{j}^{r}\right\rangle\right]^{+}
    $$
    Given a data $\Xb^{r}$, define:
\begin{align*}
    \mathcal{S}(\Xb^{r}):=\{j\in [K]: \text{the $j$-th coordinate of $\Xb^{r}$'s sparse vector $z^{r}$ is not equal to zero, i.e. }  z^{r}_{j}\neq 0 \}
\end{align*}
We abbreviate $\mathcal{S}(\Xb^{r})$ as $\mathcal{S}^{r}$ in our subsequent analyis for simplicity. We use $\mathcal{D}^{r}_{s}$ to denote the sufficient uni-modal training data for $\mathcal{M}_{r}$, and $\mathcal{D}^{r}_{i}$ for insufficient uni-modal data.

\begin{hypothesis}~\label{hypo-uni}
\begin{enumerate}[label=\roman*]
    \item [] For sufficient data $(\Xb^{r},y)\in\mathcal{D}^{r}_s$, for every $l\in[m]$:
    \item for every $j=y$, or $j\in \mathcal{S}^{r}:$ $\left\langle \nu_{j,l,r}^{(t)}, \Xb^{r}\right\rangle=\left\langle \nu_{j,l,r}^{(t)}, \Mb^{r}_{j}\right\rangle z^{r}_{j} \pm \widetilde{o}\left(\sigma_{0}\right)$.
    \item else $\left|\left\langle \nu_{j, l,r}^{(t)}, \Xb^{r}\right\rangle\right| \leq \widetilde{O}\left(\sigma_{0} \right)$
       
   \item [] For insufficient data $(\Xb^{r}, y) \in \mathcal{D}^{r}_{i}$, every $l \in[m]$:
   
       \item  for every $j =y:$ $\left\langle \nu_{j, l,r}^{(t)}, \Xb^{r}\right\rangle=\left\langle \nu_{j, l, r}^{(t)}, \Mb_{j}^{r}\right\rangle z^{r}_{j}+\left\langle \nu_{j, l, r}^{(t)}, {\xi^{r}}^{\prime}\right\rangle \pm \widetilde{O}\left(\sigma_{0} \alpha K\right)$
       \item for every $j\in \mathcal{S}^{r}:$ $\left\langle \nu_{j,l,r}^{(t)}, \Xb^{r}\right\rangle=\left\langle \nu_{j,l,r}^{(t)}, \Mb^{r}_{j}\right\rangle z^{r}_{j} \pm \widetilde{o}\left(\sigma_{0}\right)$.
      \item else
    $\left|\left\langle \nu_{j, l, r}^{(t)}, \Xb^{r}\right\rangle\right| \leq \widetilde{O}\left(\sigma_{0} \right)$
  
   
   Moreover, we have for every $j \in[K]$,
       \item $\Psi_{j}^{(t)} \geq \Omega\left(\sigma_{0}\right)$ and $\Psi_{j}^{(t)} \leq \widetilde{O}(1)$.
       \item for every $l \in[m]$, it holds that $\left\langle \nu_{j, l, r}^{(t)}, \Mb_{j}^{r}\right\rangle \geq-\widetilde{O}\left(\sigma_{0}\right)$.
   \end{enumerate}
\end{hypothesis}
\paragraph{Training phases.} The analysis for uni-modal networks with modality $\mathcal{M}_r$ can also be decomposed into two phases for each class $j\in[K]$: 
\begin{itemize}
    \item Phase 1:  $t\leq T^{r}_{j}$, where $T^{r}_{j}$ is the iteration number that $\Psi_{r,j}$ reaches $\Theta\left(\frac{\beta}{\log k}\right)=\widetilde{\Theta}(1)$ 
    \item Phaes 2, stage 1: $T^{r}_{j}\leq t\leq T^{r}_{0}$:
     where  $T^{r}_0$ denote the iteration number that all of the $\Psi^{(t)}_{r,j}$  reaches $\Theta(1/m)$;
    \item Phase 2, stage 2: $t\geq T^{r}_0$, i.e. from $T^{r}_0$ to the end $T$.
    \end{itemize}
\subsection{Main theorem for Uni-modal}
\begin{theorem}[Theorem~\ref{thm-uni} Restated]
    For every $r\in[2]$, for sufficiently large $K>0$, every $\eta \leq \frac{1}{\operatorname{poly}(k)}$, after $T=\frac{\text { poly }(k)}{\eta}$ many iteration, the learned uni-modal network ${f^{\text{uni},r}}^{(t)}$ \textit{w.h.p} satisfies:
       \begin{itemize}
           \item Training error is zero: 
           \begin{align*}
            \frac{1}{n}\sum_{(\Xb^{r},y)\in\mathcal{D}^{r} }\mathbb{I}\left\{ \right.{f_{y}^{\text{uni},r}}^{(T)}(\Xb^{r})  \leq \max_{j\neq y}  \left. {f_{j}^{\text{uni},r}}^{(T)}(\Xb^r)\right\}=0.
           \end{align*}
           \item  The test error satisfies:
            \begin{align*}
            \Pr_{(\Xb^{r},y)\sim\mathcal{P}^{r}}({f_{y}^{\text{uni},r}}^{(T)}(\Xb^r)\leq \max_{j\neq y}{f_{j}^{\text{uni},r}}^{(T)}(\Xb^r)-\frac{1}{\polylog(K)} )
            =(1 \pm o(1))\mu_r
           \end{align*}
            
   
       
       \end{itemize}

              \end{theorem}
   \begin{proof} 

        \textbf{Training error analysis.} For every data pair $(\Xb^{r},y)$, we can bound the training error in the similar manner as joint training, and obtain:
                 $$\frac{1}{T}\sum^{T}_{t=T^{r}_0}\mathcal{L}({f^{\text{uni},r}}^{(t)})=\frac{1}{T}\sum^{T}_{t=T^{r}_0}\frac{1}{n}\sum_{(\Xb,y)\in\mathcal{D}^{r}} -\log(\ell_y({f^{\text{uni},r}}^{(t)},\Xb)) \leq \frac{1}{\poly(K)}$$ 
                 Therefore,
                  $$
                  \frac{1}{n}\sum_{(\Xb,y)\in\mathcal{D}^{r}} (1-\ell_y({f^{\text{uni},r}}^{(T)},\Xb))\leq \frac{1}{n}\sum_{(\Xb,y)\in\mathcal{D}^{r}} -\log(\ell_y({f^{\text{uni},r}}^{(T)},\Xb))\leq \frac{1}{\poly(K)}
                  $$
                  Therefore, we can conclude the training error is sufficiently small at the end of the iteration $T$.
                  
                  \textbf{Test error analysis.} For the test error of ${f^{\text{uni},r}}^{(T)}$, given $j\in[K]$, we will have $c_r\Pi_{j,r}^{(T)}-\Pi_{i,r}^{(T)}\leq -\Omega(\log(K))$ for any $i,j$. Hence for sufficient data, by function approximation for uni-modal, we immediately have ${f_{y}^{\text{uni},r}}^{(T)}(\Xb^{r}) \geq \max _{j \neq y} {f_{j}^{\text{uni},r}}^{(T)}(\Xb^{r})+\Omega(\log k)$. By Induction Hypothesis~\ref{hypo-uni}, no doubt that $\mathcal{M}_r$ has been learned.  However for insufficient data, , we will have ${f_{y}^{\text{uni},r}}^{(T)}(\Xb^{r}) \leq O(\gamma_r)+\frac{1}{\polylog(K)}$ due to the data distribution. For every other $i\neq y$, as long as $i\in \mathcal{S}^{r}$, we will have ${f_{j}^{\text{uni},r}}^{(T)}(\Xb^{r}) \geq \widetilde{\Omega}(\rho_r)$. Therefore,  with probability at least  $1-e^{-\Omega\left(\log ^{2} k\right)}$, for insufficient data, we have 
                  $${f_{y}^{\text{uni},r}}^{(T)}(\Xb^{r}) \leq \max_{j\neq y}{f_{j}^{\text{uni},r}}^{(T)}(\Xb^{r})-\frac{1}{\polylog(K)}$$
                  Recall that insufficient data occurs in $\mathcal{M}_r$ with probability $\mu_r$, then we finish the proof.

             \end{proof}

\section{Experimental Setup}~\label{sec-exp}
For empirical justification, we conduct experiments on an internal product classification dataset to verify the results presented by \citet{wang2020makes}, and also provide empirical support for this theoretical analysis. 
Specifically, the dataset consists of products, each of which has an image, which is usually a photograph of the product, and a title text, which describes the key information, e.g., category, feature, etc. 
We split the dataset into two sets for training and validation. The training set consists of around $600K$ samples, and the validation set consists of $10K$ samples. 
For the evaluation of training accuracy, we sample $10K$ products from the training set. 

We build a Transformer~\citep{transformer} model for image model and text model respectively. 
Specifically, the image model is a small ViT~\citep{vit} network, consisting of $6$ transformer layers, each of which has a self attention and Feed-Forward Network (FFN) module with layer normalization and residual connection. The hidden size is $512$, and the intermediate size is $2048$. 
The image is preprocessed by resizing to the resolution of $256 \times 256$, and split into $16 \times 16$ patches. Each patch is projected to a vector by linear projection, and the patch vectors as a sequence is the input of the Transformer. 
The text model is also a Transformer model with the identical setup. Specifically, we tokenize each text with the Chinese BERT tokenizer~\citep{bert}. 
For the multi-modal late fusion model, we use the two Transformer models as bi-encoders. We element-wisely sum up their output representations, each of which is an average pooling of the Transformer outputs, and send it to a linear classifier for prediction. 

Additionally, this empirical study investigates whether the single-modal trained model can outperform a single-modal model with a fixed encoder initialized by the multi-modal model. 
This is widely used to measure self-supervised representations~\cite{chen2020simple}.
For the setup of the latter one, we build a single-modal encoder and initialize the weights with the parameters of the corresponding modality from a multi-modal model. We add a linear classifier on top and freeze the bottom encoder to avoid parameter update. 

All models are trained in an end-to-end fashion. 
We apply AdamW~\citep{adamw} optimizer for optimization with a peak learning rate of $1e-4$, a warmup ratio of $1\%$, and the cosine decay schedule. The total batch size of $256$. 
We implement our experiments on $16$ NVIDIA V100-32G. 

\end{document}